\documentclass{verts}

\usepackage[utf8]{inputenc}
\usepackage[T1]{fontenc}

\usepackage{scalerel}
\usepackage[inference]{semantic}

\usepackage{tikz}
\usetikzlibrary{cd}

\usepackage[draft,author=MS]{fixme}
\fxsetup{layout=inline,marginclue}
\fxsetup{theme=color}

\providecommand\given{}
\providecommand\and{}

\newcommand\SetSymbol[1][]{%
	\nonscript\:#1\vert
	\allowbreak\nonscript\:
	\mathopen{}}

\DeclarePairedDelimiterX\set[1]\{\}{%
	\renewcommand\given{\SetSymbol[\delimsize]}
	\def\and{\:\land\:}
	\, #1 \,
}

\DeclarePairedDelimiterX
{\norm}[1]
{\lVert}{\rVert}
{\ifblank{#1}{\,\cdot\,}{#1}}

\DeclarePairedDelimiterX
{\abs}[1]
{\lvert}{\rvert}
{\ifblank{#1}{\,\cdot\,}{#1}}

\DeclarePairedDelimiterX
{\sem}[1]
{\llbracket}{\rrbracket}
{\ifblank{#1}{\,\cdot\,}{#1}}

\DeclarePairedDelimiterX
{\innerprod}[2]
{\langle}{\rangle}
{\ifblank{#1}{\,\cdot\,}{#1},\ifblank{#2}{\,\cdot\,}{#2}}

\DeclarePairedDelimiterX
{\repr}[1]
{[}{]}
{\ifblank{#1}{\,\cdot\,}{#1}}

\DeclareMathOperator{\sign}{sgn}

\DeclareMathOperator{\classifier}{\nu_c}

\DeclareMathOperator*{\argmax}{arg\,max}

\newcommand{\natset}{\mathbb{N}}

\newcommand{\realset}{\mathbb{R}}

\newcommand{\affset}[2]
{\ifblank{#1}
	{\ifblank{#2}{\Phi}{}}
	{\Phi^{{#1}\to{#2}}}}
\newcommand{\tadsset}[2]
{\ifblank{#1}
	{\ifblank{#2}{\Theta}{}}
	{\Theta^{{#1}\to{#2}}}}
\newcommand\pafset{\mathrm{PAF}}
\newcommand{\netset}{\mathcal{N}}

\newcommand{\aff}{\alpha}
\newcommand{\paff}{\psi}

\newcommand{\defined}{\mathrel{\vcentcolon=}}

\newcommand{\mat}[1]{\bm{#1}}

\newcommand{\transpose}{\top}

\newcommand{\inftynorm}[1]{\lvert \! \lvert {#1} \rvert \!\rvert_{\infty}}

\newcommand\netcomp{\mathbin{;}}

\newcommand{\netsem}[1]{\sem{#1}_{\scriptscriptstyle\mathcal{N}}}
\newcommand{\netclass}{\mathcal{C}}

\newcommand{\tadssem}{\semtads}
\newcommand{\tadsjoin}{\Join}

\newcommand\restr[2]{{
  \left.\kern-\nulldelimiterspace 
  #1 
  \vphantom{\big|} 
  \right|_{#2} 
  }}

\newcommand{\eqeos}{\rlap{~~.}}

\newcommand{\transition}[1]{\mathrel{\xrightarrow{#1}}}

\definecolor{llgray}{rgb}{0.93, 0.93, 0.93}
\newcommand{\highlight}[1]{%
	\begin{center} \fcolorbox{gray}{llgray}{\parbox{\columnwidth-15pt}{\strut #1 \strut}} 
	\end{center}}

\newcommand{\semtads}[1]{\sem{#1}_{\scriptscriptstyle\Theta}}

\begin{document}
\title{The Power of Typed Affine Decision Structures: A Case Study}
\titlerunning{The Power of Typed Affine Decision Structures}
\authorrunning{G. Nolte et al.}
\author{
	Gerrit Nolte%
	\setorcid{0000-0002-5080-1039} \inst{1} \eqcontr \and
	Maximilian Schlüter%
	\setorcid{0000-0002-5100-7259} \inst{1} \eqcontr \and
	Alnis Murtovi%
	\inst{1} \and
	and Bernhard Steffen%
	\setorcid{0000-0001-9619-1558} \inst{1}
}

\institute{TU Dortmund University, Dortmund, Germany}

\corresponding{gerrit.nolte@tu-dortmund}

\abstract{
TADS are a novel, concise  white-box representation of neural networks.
In this paper, we apply TADS to the problem of neural network verification, using them to generate either proofs or concise error characterizations for desirable neural network properties.
In a case study, we consider the robustness of neural networks to adversarial attacks, i.e., small changes of an input that drastically change a neural networks perception, and show that TADS can be used to provide precise diagnostics on how and where robustness errors a occur.
We achieve these results by introducing \emph{Precondition Projection}, a technique that yields a TADS describing network behavior precisely on a given subset of its input space, and combining it with PCA, a traditional, well-understood dimensionality reduction technique.
We show that PCA is easily compatible with TADS.
All analyses can be implemented in a straightforward fashion using the rich algebraic properties of TADS, demonstrating the utility of the TADS framework for neural network explainability and verification.
While TADS do not yet scale as efficiently as state-of-the-art neural network verifiers, we show that, using PCA based simplifications, they can still scale to medium-sized problems and yield concise explanations for potential errors that can be used for other purposes such as debugging a network or generating new training samples.
}%


\maketitle

\section{Introduction}%
\label{sec:Introduction}%
In recent years, neural networks have been a driving force behind many of the most exciting success stories in machine learning. From image recognition \cite{gu2018recent} and speech recognition \cite{brown2020language} to playing complex games on a superhuman level \cite{vinyals2019grandmaster}, neural networks have achieved results that were almost unthinkable even a decade ago. 

However, while the size, performance and scope of neural networks steadily increases, their opaqueness remains an equally important and essentially unsolved problem~\cite{adadi2018peeking}. Frequently denoted as ``black-box''-models, the decisions of neural networks are to this day hard to explain and, likewise, their properties hard to verify.

In this paper, we are concerned with Typed Affine Decision Structures (TADS)~\cite{schlueter2022towards}, a novel decision-tree-like data structure that represents piece-wise affine functions. TADS are specifically designed to act as interpretable white-box models that can precisely represent any piece-wise linear neural network in an understandable fashion.

While TADS are structurally well-suited for global model explanation and verification of neural networks, the full explanation of even medium-sized neural networks is well out of scope. It is well-known that the semantic complexity of a neural network---with respect to many different measures of complexity---grows exponentially in its size. As a consequence, any precise global explanation of such a model incurs exponential scaling issues \cite{bianchini2014complexity,montufar2014number,fazlyab2019efficient}.

In this paper, we are interested in applying TADS to verifying local properties of neural networks, most notably robustness properties \cite{carlini2017towards,luo2018towards,szegedy2013intriguing,mothilal2020explaining}. Robustness properties encode that, at certain points that a user desires, a neural network's classification is invariant to small changes of its input. For example, in image recognition, if one knows that an image represents a certain object, a single flip of a pixel should not drastically change the networks \textit{correct} classification of said image. Robustness properties are the most commonly considered properties in neural network verification and make up the majority of current benchmarks in the VNNComp verification competition \cite{bak2021second}.

To apply TADS, which are in principle global model explanations, to local properties, we will introduce precondition projection, a transformation of TADS that restricts their domain to a certain region of interest. Further, we show how the algebraic properties of TADS can be used to directly model the classification behavior of a neural network. This is important as neural networks, although often used as classifiers (assigning one of finitely many classes to an input), are fundamentally regression models (assigning real values to their input). By modeling the argmax function directly on a TADS level, this gap can be bridged in an elegant fashion.

Finally, we will present a case study in which we apply TADS to robustness analysis and present its advantages. At present, TADS do not yet scale well to larger problems. We will introduce an approach that uses the well-understood dimensionality reduction technique PCA to prove an underapproximation to the robustness property of interest. This approach mitigates the scaling issues incurred by TADSs, but lacks reliable guarantees on robustness. Thus, we introduce another approach that directly trains neural networks to operate on inputs that are simplified by PCA\@. This method is of similar computational complexity than the underapproximation approach, but yields neural networks for which TADS can give reliable robustness guarantees, while incurring only a small loss in neural network accuracy. 

Lastly, we will show on a concrete example how a TADS-based robustness proof looks like and what additional information it yields beyond already existing verification tools. We will show how this information can be used to characterize precisely and completely the entire set of inputs that violate a given property, and how it can be used to find ``closest'' adversarial examples, if they exist.  

\section{Preliminaries}
\label{sec:Prelims}
\subsection{Linear Algebra and Notation}
\label{subsec:linalg}

The following notations of linear algebra are based on the book~\cite{axler1997linear}.
The real vector space $(\realset^n, +, \cdot)$ with $n > 0$ is an algebraic structure with the operations
\begin{align*}
	+ &\colon \realset^n \times \realset^n \to \realset^n & &\text{vector addition}\\
	{}\cdot{} &\colon \realset\phantom{{}^n} \times \realset^n \to \realset^n & &\text{scalar multiplication}
\end{align*}
which are defined as
\begin{align*}
	(x_1,\dots, x_n) + (y_1,\dots,y_n) &= (x_1 + y_1, \dots, x_n + y_n) \\
	\lambda \cdot (x_1,\dots,x_n) &= (\lambda \cdot x_1, \dots, \lambda \cdot x_n) \eqeos
\end{align*}
A real vector $(x_1,\dots,x_n)$ of $\realset^n$ is abbreviated as $\vec x$.
To refer to its $i$-th component, we write $x_i$ (in contrast, $\vec{x_i}$ denotes the $i$-th \emph{vector} in some enumeration).
The dimension of a real vector space $\realset^n$ is given as ${\dim \realset^n = n}$.

A matrix $\mat W$ is a collection of real values arranged in a rectangular array with $n$ rows and $m$ columns.
\[ \mat W = 
\begin{pmatrix}
	w_{1,1} &w_{1,2} &\dots  &w_{1,m} \\
	w_{2,1} &w_{2,2} &\dots  &w_{2,m} \\
	\vdots  &\vdots  &\ddots &\vdots \\
	w_{n,1} &w_{n,2} &\dots  &w_{n,m}
\end{pmatrix} \]
To indicate the number of rows and columns, one says $\mat W$ has \emph{type} $n \times m$ commonly  notated as $\mat W \in \realset^{n \times m}$.

An element at position $i,j$ of the matrix $\mat{W}$ is denoted by $\mat{W}_{i,j} \defined w_{i,j}$ (where $1 \leq i \leq n$ and $1 \leq j \leq m$).
A matrix $\mat W \in \realset^{n\times m}$ can be reflected along the main diagonal resulting in the transpose $\mat{W}^\transpose$ of shape $m \times n$ defined by the equation
\[ \big(\: \mat{W}^\transpose \:\big)_{i,j} \defined \mat{W}_{j,i} \eqeos \]
The $i$-th row of $\mat W$ can be regarded as a $1 \times m$ matrix given by
\[ \mat{W}_{i,\bullet} \defined (w_{i,1}, \dots, w_{i,m}) \eqeos\]
Similarly, the $j$-th column of $\mat W$ can be regarded as a $n \times 1$ matrix defined as
\[\mat{W}_{\bullet,j} \defined (w_{1,j}, \dots, w_{n,j})^\transpose \eqeos\]
Matrix addition is defined over matrices with the same type to be component-wise, i.e., 
\[ \big(\: \mat{W} + \mat{N} \:\big)_{i,j} \defined \mat{W}_{i,j} + \mat{N}_{i,j} \]
and scalar multiplication as
\[ \big(\: \lambda \cdot \mat{W} \:\big)_{i,j} \defined \lambda \cdot \mat{W}_{i,j} \eqeos \]
The (type-correct) multiplication of two matrices \linebreak $\mat W \in \realset^{n\times r}$ and $\mat N \in \realset^{r\times m}$ is defined as
\[ \big(\: \mat W \cdot \mat N \:\big)_{i,j} \defined \sum_{k=1}^r \mat{W}_{i,k}\cdot \mat{N}_{k,j} \eqeos \]
Identifying 
\begin{itemize}
	\item $n \times 1$ matrices with (column) vectors
	\item $1 \times m$ matrices with row vectors
	\item $1 \times 1$ matrices with scalars
\end{itemize}
as indicated above, makes the well-known dot product of $\vec v, \vec w \in \realset^n$
\[ \innerprod{\vec v}{\vec w} \defined \sum_{i=1}^{n} \vec{v}_i \cdot \vec{w}_i \]
just a special case of matrix multiplication.
The same holds for matrix-vector multiplication that is defined for a $n \times m$ matrix $\mat W$ and a vector $\vec x \in \realset^n$ as
\[ \big(\: \mat W \cdot \vec x \:\big)_{i} \defined (\mat W_{i,\bullet}) \cdot \vec x \eqeos \]
Matrices with the same number of rows and columns, i.e., with type $n\times n$ for some $n\in\natset$, are said to be \emph{square matrices}.

\subsection{Affine Functions}
\begin{definition}[Affine Function]\label{def:affine_functions}%
	A function $\aff\colon \realset^{n} \rightarrow \realset^m$ is called \emph{affine} iff it can be written as 
	\[	\aff(\vec x)= \mat{W} \vec x + \vec b   \]
	for some matrix $\mat W \in \realset^{m \times n}$ and vector $b \in \realset^m$.
	We identify the semantics and syntax of affine functions with the pair $(\mat W, \vec b)$ which can be considered 
	as a canonical representation of affine functions.
	Furthermore, we denote the set of all affine functions $\realset^n \rightarrow \realset^m$ as $\affset{n}{m}$
	with \emph{type} $(n,m)$.
	The untyped version $\affset{}{}$ is meant to refer to the set of all affine functions, independently of their type.
\end{definition}
\begin{lemma}[Operations on Affine Functions]\label{theo:affine_operations}%
	Let $\aff_1,\aff_2$ be two affine functions in canonical form, i.e., 
	\begin{align*}
		\aff_1(\vec x) &= \mat{W_1} \vec x + \vec{b_1} \\
		\aff_2(\vec x) &= \mat{W_2} \vec x + \vec{b_2} \eqeos
	\end{align*}
	Assuming matching types, the operations $+$ (addition), $\cdot$ (scalar multiplication), and $\circ$ (function application) can be calculated on the representation as
	\begin{align*}
		(s \cdot \aff_1)(\vec x)
		&= (s \cdot \mat {W_1}) \,\vec x + (s \cdot \vec {b_1}) \\
		(\aff_1 + \aff_2)(\vec x)
		&= (\mat {W_1}+\mat {W_2}) \,\vec x + (\vec {b_1} + \vec {b_2}) \\
		(\aff_2 \circ \aff_1)(\vec x)
		&= (\mat {W_2} \mat {W_1}) \,\vec x + (\mat {W_2} \vec {b_1} + \vec {b_2})
	\end{align*}
	resulting again in an affine function in canonical representation.
\end{lemma}
It is well-known that the type resulting from function composition evolves as follows
\[ \circ\colon \affset{r}{m} \times \affset{n}{r} \to \affset{n}{m} \eqeos \]
The type of the operation is important for the closure axiom, the basis for most algebraic structures. 
This leads to the following well-known theorem \cite{axler1997linear}:
\begin{theorem}[Algebraic Properties]%
	\label{prop:algprop}%
	Denoting, as usual, scalar multiplication with $\cdot$ and \linebreak
	function composition with $\circ$, we have:
	\begin{itemize}
		\item
		$(\affset{n}{m}, +, \cdot) $ is a vector space and
		\item 
		$(\affset{n}{n}, \circ, \mathrm{id}) $ is a monoid. 
	\end{itemize}
\end{theorem}
This theorem can straightforwardly be lifted to untyped $\affset{}{}$ by simply restricting all operations 
to the cases where they are well-typed, i.e., where addition  is restricted to functions of the same type ($+_t$), and 
function composition to situation where the output type of the first function matches the input type of the second ($\circ_t$): 

\begin{theorem}[Properties of Typed Operations]%
	\label{prop:typed-algprop}%
	$(\affset{}{}, +_t, \cdot, \circ_t) $ is a typed algebra, i.e, an algebraic  \linebreak structure that is closed under well-typed operations.
\end{theorem}
\subsection{Piece-wise Affine Functions}
Piece-wise affine functions (PAFs) are studied extensively in tropical geometry (introductory book \cite{maclagan2021tropical}, in context of machine learning \cite{maragos2021tropical}), are used in interpolation (given their strong connection to splines and Riemann integrals), and are more and more analyzed with respect to their connection to neural networks 
\cite{Hinz2021,chu2018exact,montufar2014number,gehr2018ai2,pascanu2013number,zhang2020empirical,zhang2018tropical,Hanin2019,serra2018bounding,hanin2019complexity,Woo2018,Sudjianto2020UnwrappingTB,raghu2017expressive,arora2016understanding} (specifically to PLNNs, see \cref{def:nn}).
PAFs are usually defined over a polyhedral partitioning of the pre-image space \cite{Brondsted93,GorokhovikZ94,Ovchinnikov2010}.
Polyhedra arise by intersecting halfspaces:
\begin{definition}[Hyperplanes and Halfspace]
Let $\vec w \in \realset^n$ and $b \in \realset$.
Then the set 
\[ p = \set{\vec x \in \realset^n \given \innerprod{\vec w}{\vec x} + b = 0} \]
is called a \emph{hyperplane} of $\realset^n$.
A hyperplane partitions $\realset^n$ into two convex subspaces, called \emph{halfspaces}.
The positive and negative halfspaces of $p$, respectively, are defined as
\begin{align*}
	p^+ &\defined \set{\vec x \in \realset^n \given \innerprod{\vec w}{\vec x} + b > 0 } \\
	p^- &\defined \set{\vec x \in \realset^n \given \innerprod{\vec w}{\vec x} + b < 0 } \eqeos
\end{align*}
\end{definition}
\begin{definition}[Polyhedron]
	A polyhedron $Q \subseteq \realset^n$ is the intersection of $k$ halfspaces for some
	natural number $k$
	\[ Q = \bigcap_{i=1}^k \set{ \vec x \in \realset^n \given \innerprod{\vec {w_i}}{\vec x} + {b_i} \leq 0 } \eqeos \]
\end{definition}

\begin{definition}[Piece-wise Affine Function]
	A function $\paff\colon \realset^{n} \rightarrow \realset^m$ is called \textit{piece-wise affine} if it can be written as
	\[ \paff(\vec x)= \aff_i(\vec x) \; \text{ for } \; \vec x \in Q_i \]
	where $Q = \set{Q_1,\dots,Q_k}$ is a set of polyhedra that partitions the space $\realset^n$ and $\aff_1, \dots, \aff_k$ are affine functions. We call $\aff_i=\mat W_i \vec x +\vec b_i$ with $1\leq i\leq k$ the function associated with polyhedron $Q_i$.
\end{definition}

\subsubsection{Norms and Distances}
\label{subsec:norms}

\begin{figure}[htp]
	\centering
	\subfloat[$B_1^{2}$]{%
		\includegraphics[width=.48\columnwidth]{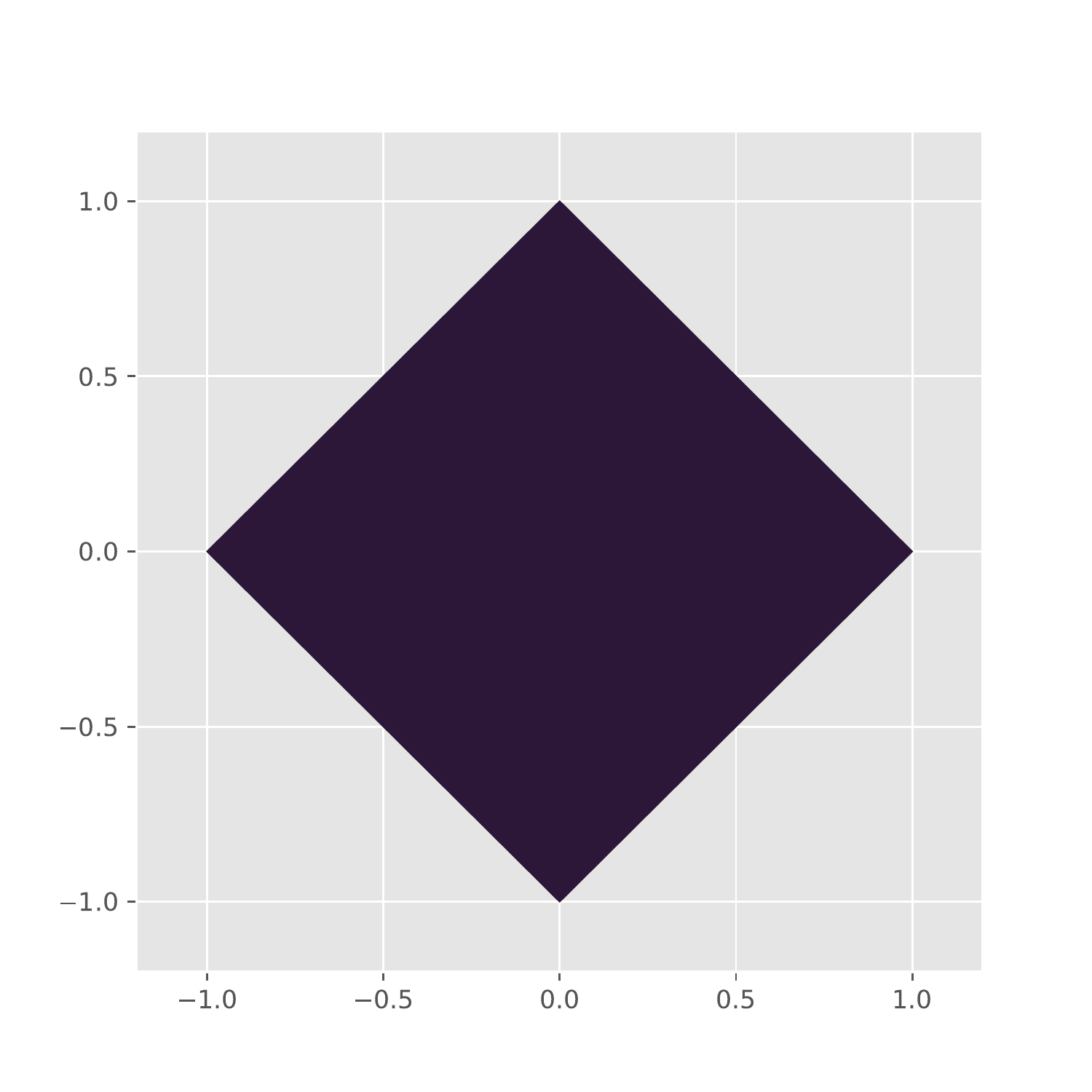}
		\label{fig:sfig1}
	} 
	\subfloat[$B_2^{2}$]{%
		\includegraphics[width=.48\columnwidth]{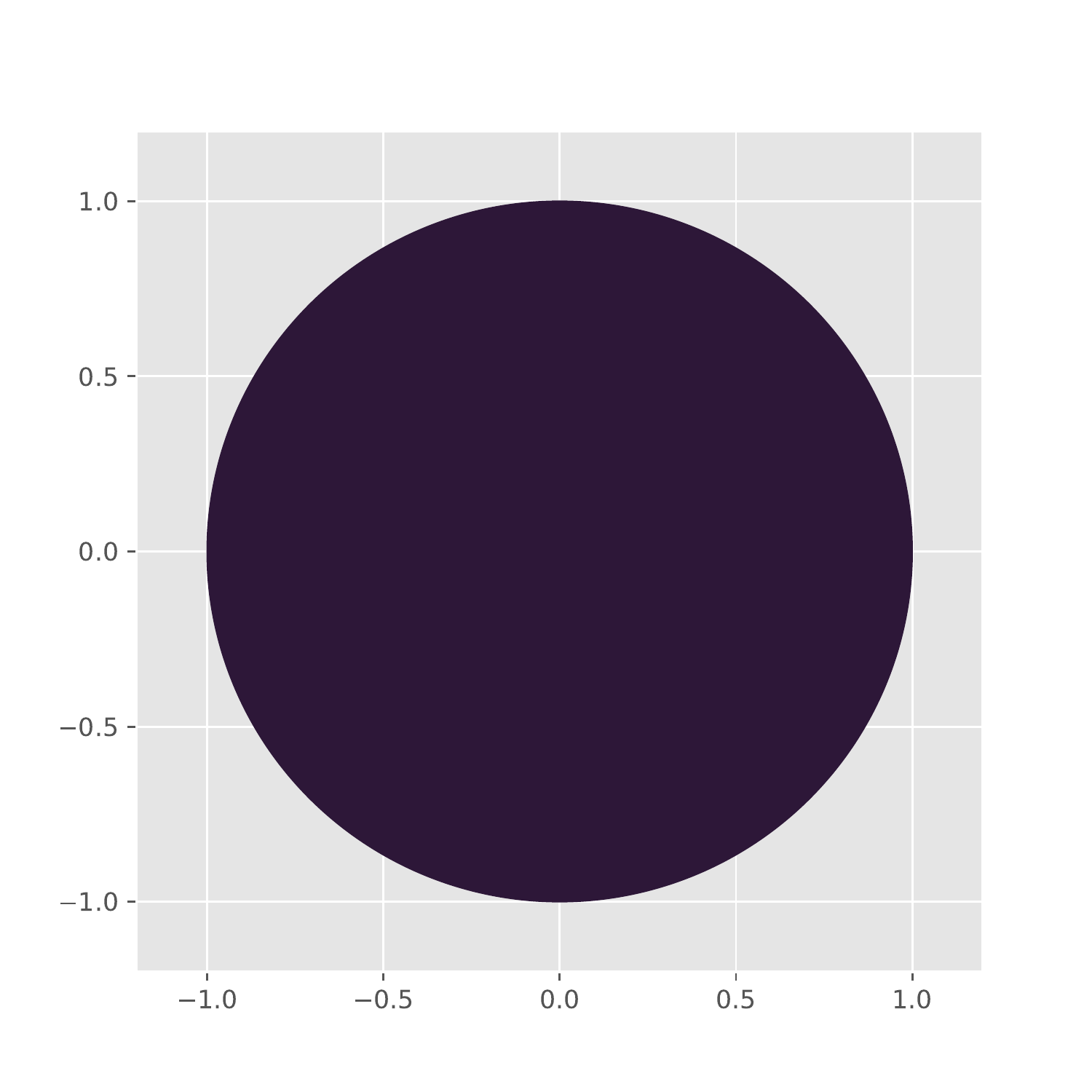}
		\label{fig:sfig2}
	} \\ 
	\subfloat[$B_\infty^{2}$]{%
		\includegraphics[width=.48\columnwidth]{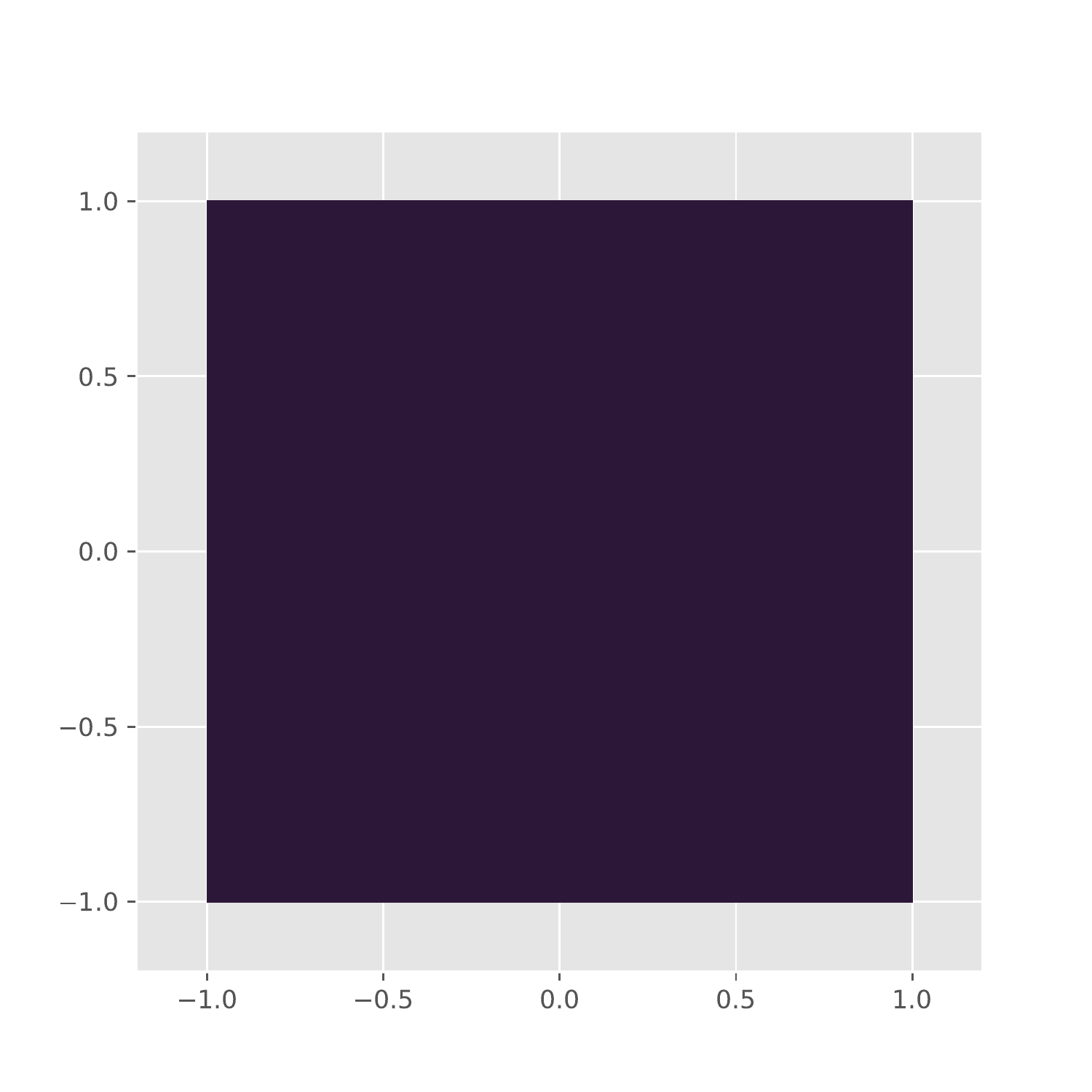}
		\label{fig:sfig3}
	}
	\caption{The two-dimensional $m$-balls $B_m^{2}$ for $m=1$ (top left), $m=2$ (top right) and $m=\infty$ (bottom).}
	\label{fig:lballs}
\end{figure}
Throughout this work, we will often be concerned with the behavior of neural networks and how it changes when a point is slightly altered. Thus, we will often be concerned with different \emph{neighborhoods} of points.
These are formalized in mathematics using metric spaces and normed spaces \cite{magnus2022metric}.
For our purposes, however, a special type of normed spaces defined by so-called $l$-norms is sufficient:

\begin{definition}[L-Norms]
For $m \in \natset$, the $l_m$-norm is the function ${\norm{\cdot}_m \colon \realset^n \rightarrow \realset}$ defined by:
\[ \norm{\vec x}_m \defined \sqrt[m]{\sum_{i=1}^n \abs{x_i}^m} \eqeos \]
\end{definition}
Important $l$-norms are the $l_1$ norm
\[ \norm{\vec x }_1 \defined \sum_{i=1}^n \abs{x_i} \]
and the $l_2$ norm or euclidean norm\footnote{For real vector spaces $\realset^n$ the euclidean norm is the canonical norm as $\norm{\vec x}_2 = \innerprod{\vec x}{\vec x}$.}
\[ \norm{\vec x }_2{}^2 \defined \sqrt{\sum_{i=1}^n x_i^2} \eqeos \]
Another important norm is the so-called $l_{\infty}$ norm. While not technically an $l$-norm according to the definition above, it arises naturally as the limit of the $l$-norms as $m$ approaches infinity and is defined by:
\[ \norm{\vec x}_{\infty} \defined \max\nolimits_{i \in \{1,\ldots, n \}} \abs{x_i} \]
With these definitions we can now formalize the neighborhood of a point.
\begin{definition}[Unit Ball]\label{def:mball}
	For a given $l_m$-norm $\norm{}_m$ we define the corresponding $m$-unit ball as
	\begin{equation}\label{eq:m_ball}
		B^n_m \defined \set{ \vec x \in \realset^n \given \norm{\vec x}_m \leq 1} \eqeos
	\end{equation}
\end{definition}
The unit ball is a closed, convex subset of $\realset^n$ centered at the origin.
It generalizes the notion of a disk with radius 1 to both higher dimensions ($n > 2$) and non-euclidean spaces ($m \neq 2$).
The relevant unit balls for this paper are illustrated in \cref{fig:lballs}.
Every generic $m$-ball of $\realset^n$ can be constructed using translation
\[ \vec y + B^n_m = \set{ \vec x \in \realset^n \given \norm{\vec x - \vec y}_m \leq 1} \]
and scaling
\[ r B^n_m = \set{ \vec x \in \realset^n \given \norm{\vec x}_m \leq r} \]
of unit $m$-balls.
In the latter case $r$ is called the radius.
If it is clear from the context, we omit the dimensionality of the $m$-ball.

\subsection{Neural Networks}
\label{subsec:NN}
The following brief introduction to neural networks is based on \cite{Goodfellow-et-al-2016}, but in its presentation adapted to better fit the context of this work. 

Neural networks are perhaps todays most important machine learning models that are most succinctly characterized by their layered structure. There exist numerous neural network architectures that one might consider. For this work, we focus on the very general class of fully connected neural networks and define neural networks as follows:
\begin{definition}[Piece-wise Linear Neural Networks]%
\label{def:nn}%
A piece-wise linear neural network $\nu$ with $l$ layers is a machine learning model consisting of an alternating sequence of $l$ \emph{affine preactivation functions} $\aff_i$ and ${l-1}$ \emph{ReLU activation functions} $\phi$:
\[ \nu = \aff_{l+1} \netcomp \phi \netcomp \aff_l \netcomp \dots \netcomp \phi \netcomp \aff_1 \eqeos \]
For the PLNN to be syntactically correct, the affine functions must be compatible, i.e., the output dimension of each preactivation must match the input dimension of the following.
\end{definition}

In accordance to standard neural network terminology, we call the combination of a preactivation with its activation $\phi \circ \alpha_i$ the $i$-th layer of the neural network.

In the following paragraphs preactivations and activations are properly introduced.
After that, the semantics of a PLNN can be defined in terms of its components.
Lastly, a common complexity measure of PLNN is presented.

\paragraph*{Preactivations.}\ 
In traditional applications, the concrete affine functions $\alpha_k$ of a PLNN $\nu$, as defined in \cref{def:nn}, would result from a training process, usually using gradient descent based optimization techniques~\cite{Goodfellow-et-al-2016}, where the PLNN is trained to accurately predict desired outputs on a given dataset.

\paragraph*{Activations.}\ The activation function $\phi$ is an architectural design choice made a-priori by the user.
The primary purpose of the activation function is to introduce \emph{non-linearity} into the neural network, which can drastically increase the amount of functions that can be approximated.
For the purposes of this paper, we exclusively use the rectified linear unit (ReLU) function. 

\paragraph*{ReLU.}\ The ReLU function $\phi$ has proven to be a successful activation function in practical applications, combining convenient properties of linear functions with a sufficient degree of non-linearity. It is prominently recommended as the default choice of activation function for fully connected neural networks~\cite{Goodfellow-et-al-2016}.
Furthermore, due to the simple structure of the ReLU function, neural networks with ReLU activations lend themselves well to formal analysis and are typically considered in verification tasks \cite{katz2017reluplex,bak2021second}.

\begin{figure}
	\centering
	\includegraphics[width=0.7\columnwidth]{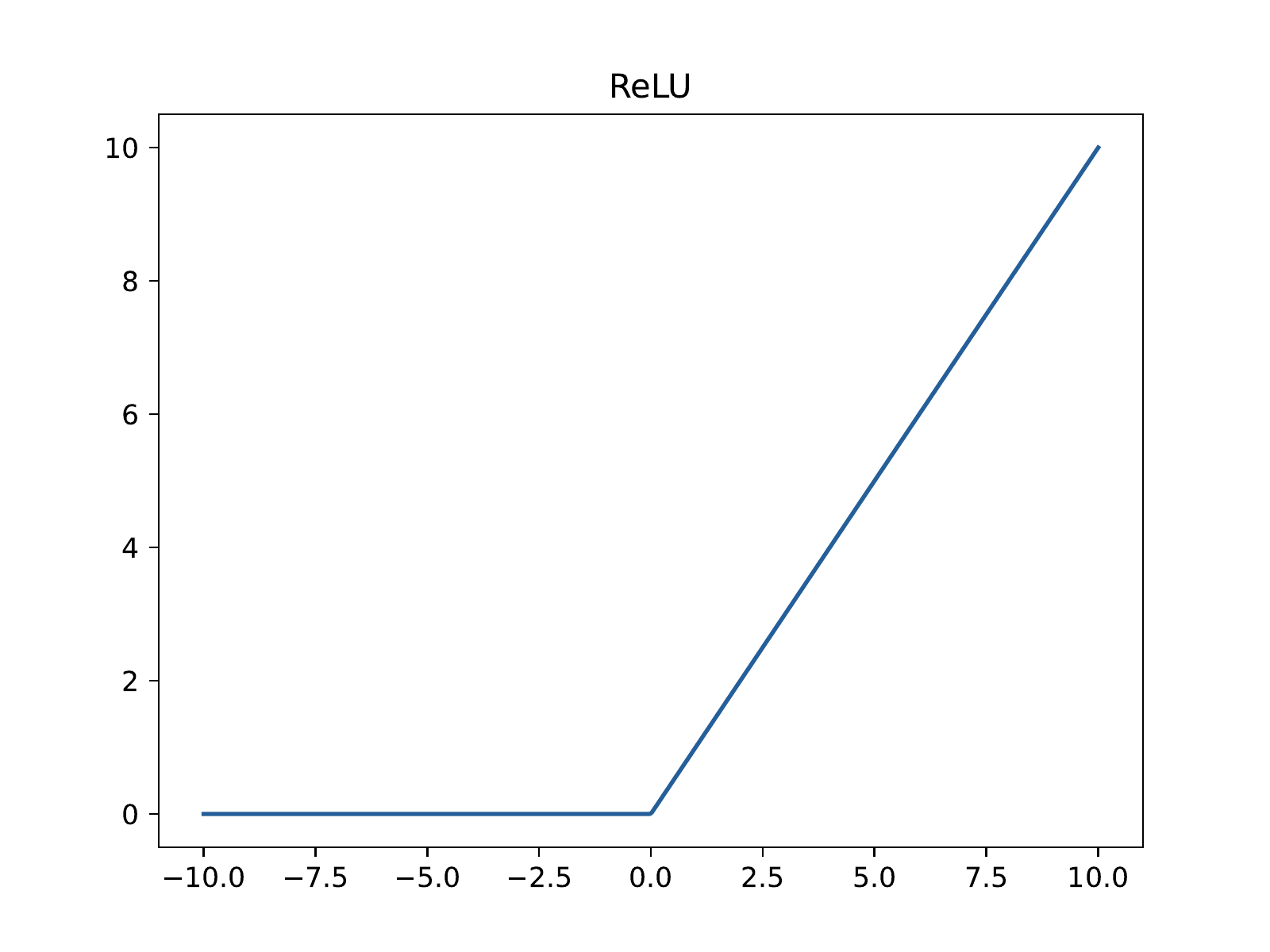}
	\caption{The function plot of the ReLU function.}
	\label{fig:relu}
\end{figure}

\begin{definition}[Rectified Linear Unit]%
	\label{def:relu}%
	The one-dimensional ReLU function ${\phi_1 \colon \realset \rightarrow \realset}$ is defined as the positive part of its argument:
	\[ \phi_1(x)= \max(0,x) \ . \]
	The $k$-dimensional ReLU function ${\phi_k \colon \realset^k \rightarrow \realset^k}$ is the elementwise application of the one-dimensional ReLU function:
	\[ \phi_k \big( (x_1, x_2,\dots,x_k)^\transpose \big) = \big( \phi_1(x_1), \phi_1(x_2),\dots, \phi_1(x_k) \big)^\transpose \eqeos \]
\end{definition}
As the ReLU activation function is the only activation function we consider, we will use $\phi$ for the remainder of this paper exclusively to refer to the ReLU function and omit the explicit mention of the dimensionality when it is clear from context. 

\begin{definition}[PLNN Semantics]%
	\label{def:plnn-semantis}%
	The semantics of a piece-wise linear neural network $\nu$ is a piece-wise affine function ${\netsem{\nu} \colon \realset^n \rightarrow \realset^m}$ given by the sequential evaluation of its layers:
	\[ \netsem{\nu} = \aff_{l+1} \circ \phi \circ \aff_l \circ \dots \circ \phi \circ \aff_1 \eqeos \]
\end{definition}
For evaluation, a vector $\vec{x}\in\realset^n$ is passed layer by layer through the PLNN.
The data-flow is unidirectional and using the above notation from right to left.

Note that, given the close relationship between a PLNN's syntax and semantics, many works in deep learning choose to not clearly separate the  syntax and semantics of PLNN's. A transition between the two definitions can be easily achieved by replacing ';' with '$\circ$' in \cref{def:nn}.

\begin{figure}
	\centering
	\includegraphics[width=.8\columnwidth]{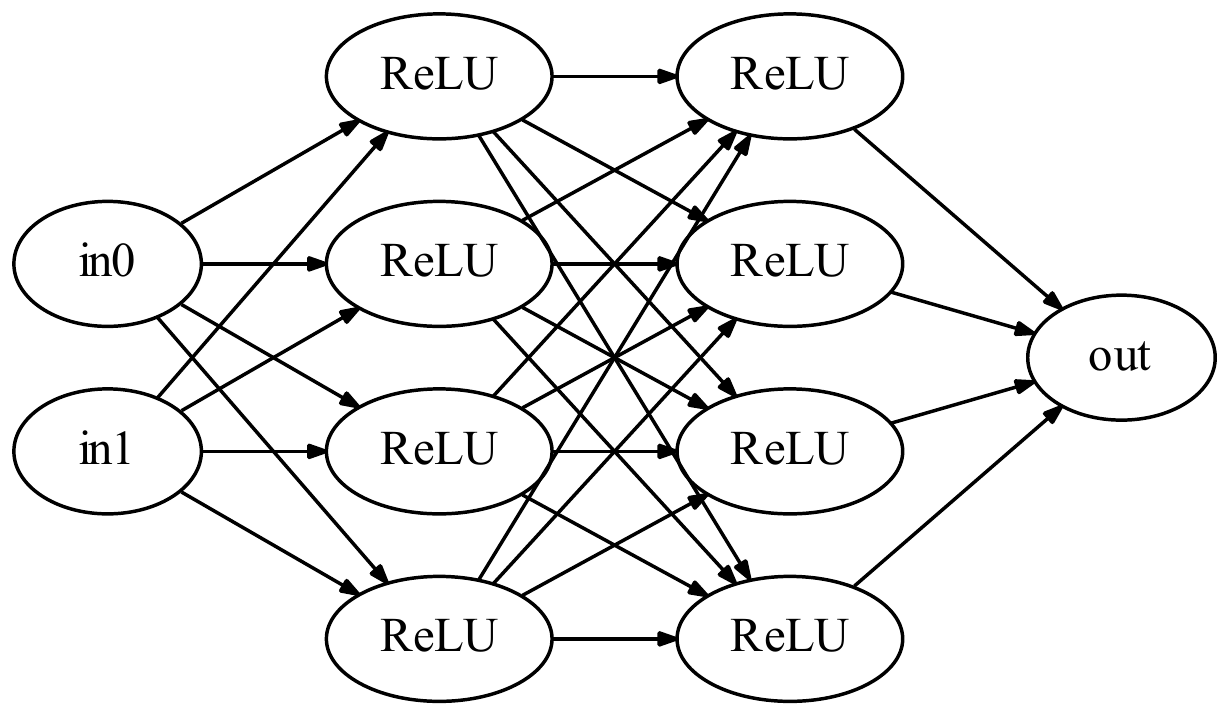}
	\caption{A simple PLNN with two hidden layers and ReLU activations.}
	\label{fig:plnn_example}
\end{figure}

Traditionally, neural networks are visualized as computation graphs, where the nodes are the eponymous ``neurons''. 
There, each affine function $\alpha_i \colon \realset^n \rightarrow \realset^m$ is visualized as a bipartite graph connecting $n$ input neurons to $m$ output neurons.
An example of such graph is given in \cref{fig:plnn_example}.

From the representation used in \cref{def:nn}, the number of neurons of a neural network can be computed through the preactivations as follows:
Let $\nu = \aff_{l+1} \netcomp \dots \netcomp \aff_1$ with $\aff_i\colon\realset^{n_i}\to\realset^{n_{i+1}}$ then the total number of neurons of $\nu$ is given by
\[ \sum_{i=1}^{l+1} n_{i+1} \eqeos \]
The number of neurons is a natural measure of ``size'' in a neural network, and it is well-known that the semantic complexity of functions---measured in the number of linear regions that are needed to characterize them---that a neural network can represent increases exponentially in its number of neurons~\cite{bianchini2014complexity,montufar2014number,fazlyab2019efficient}.

\subsubsection{Neural Networks Classifiers}
\label{sec:nnclassification}
As defined in \cref{def:nn}, PLNNs are fundamentally representations of continuous functions $\realset^n \rightarrow \realset^m$. However, they are frequently employed in classification tasks where the co-domain is instead a discrete set of classes $\{1,\dots, c\}$. To bridge this gap, one typically proceeds by training a neural network $\nu \colon \realset^n \rightarrow \realset^c$ and associating each component $y_i$ of its output $\vec y=\nu(\vec x)$ with the $i$-th class. Then, the class with the largest $y_i$ is chosen for classification.

This is formalized by the argmax function.
\begin{definition}[Argmax]
The $k$-dimensional argmax function 
\[ \argmax\nolimits_k\colon \realset^k \rightarrow \{1,\dots, k\} \]
is defined as
\[ \argmax(x_1,\dots, x_k) = j \] iff $j$ is the smallest index for which $x_j \geq x_i$ holds for all $1 \leq i \leq k$. 
\end{definition}
Again, when it is clear from context, we omit the index denoting the dimensionality and simply write $\argmax$. 

As described before, the argmax function can be used to convert PLNNs into classifiers. This naturally leads us to define PLNN classifiers.

\begin{definition}[PLNN Classifiers]
For a PLNN $\nu$ with $\sem{\nu} \colon \realset^n \rightarrow \realset^m$, the corresponding \emph{PLNN classifier} $\nu_c \colon \realset^n \rightarrow \{1,\dots, m\}$ is defined as
\[ \nu_c = {\argmax} \circ {\netsem{\nu}} \eqeos \]
\end{definition}
\subsection{Typed Affine Decision Structures}
\label{subsec:TADS}

Central to our explanation approach is a decision-tree-like data structure that we call \emph{Typed Affine Decision Structure} (TADS).
Based on the transformation process presented in \cite{schlueter2022towards}, it is possible to transform a PLNN $\nu$ into a \emph{semantically equivalent} TADS $\theta(\nu)$.
The transformation is based on the common syntactical representation of PLNNs and is compositional in the layers.

PLNNs explanation and verification is challenging because of the complex data flow of PLNNs \cite{schlueter2022towards}.


\begin{figure}
	\centering
	\includegraphics[width=.7\columnwidth]{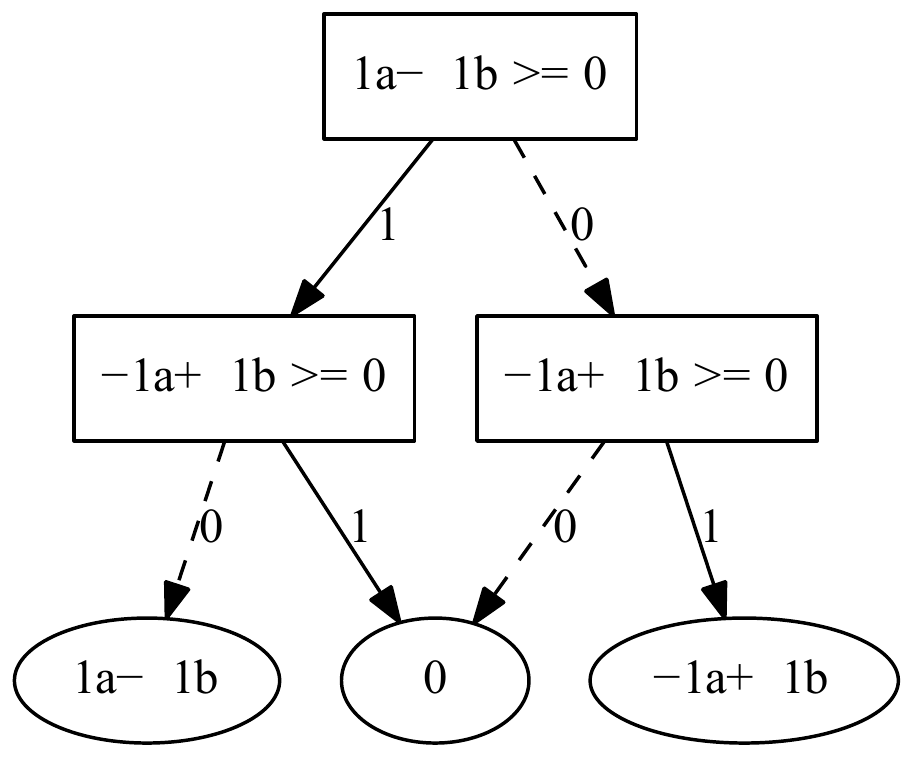}
	\caption{A simple TADS implementing the piece-wise affine function $\abs{x_1 - x_2}$.}
	\label{fig:tads_example}
\end{figure}

\paragraph*{Data Structure.}\ Skipping implementation details, TADS can be introduced intuitively using decision trees.
In a decision tree, one distinguishes two types of nodes: 
\begin{enumerate}
	\item Inner nodes have decision predicates. For every possible evaluation of that predicate, the node has exactly one successor.
	\item Leaves are elements from a given universe that one wants to distinguish.
\end{enumerate}
For TADS, specifically, leaves are from the universe of affine functions and decision predicates are affine inequalities\footnote{Also called inhomogeneous inequalities.}.
An example of a TADS can be found in \cref{fig:tads_example}.
TADS structurally resemble decision trees that are widely considered explainable machine learning models, i.e., they can, by virtue of their structure, be understood by a human\cite{guidotti2018explaining}.

Based on this introduction using decision trees one can straightforwardly define TADS. 
\begin{definition}[TADS]
	A TADS $t = (N, \rightarrow, \zeta)$ is a decision DAG\footnote{A decision DAG is a decision tree where isomorphic subtrees have been merged, see \cite{gossen2021algebraic}.} $(N, \rightarrow, \zeta)$ with root $\zeta$ whose nodes $N$ have the following two types:
	\begin{enumerate}
		\item Inner nodes are called decisions or predicates. They consist of an affine inequality and two successors, one if the predicate is true and one if not.
		\item Leaves are also called terminals. They are affine functions and have no successors.
	\end{enumerate}
	To be syntactically correct, all nodes (i.e., all inequalities and affine functions) must accept input vectors with a fixed number of entries.
	This is called the \emph{input dimension} of the TADS\@.
	Similarly, all terminals must map input vectors into a common output space.
	The dimensionality of this output space is called the \emph{output dimension}.\footnote{Note that the input and output vectors always result from the real vector spaces $\realset^n$ and $\realset^m$, respectively. Therefore, the number of entries is well-defined and equal to the dimensionality of the full vector space of inputs and outputs.}
	For given input dimension $n$ and output dimension $m$ we define the set of all TADS as $\tadsset{n}{m}$.
\end{definition}
TADS are sequentially evaluated like a decision tree.
\begin{definition}[TADS Evaluation]
	The semantic function of TADS\footnote{In this definition we use \emph{currying}, as known in, e.g., functional programming.}
	\[ \semtads{} \colon \tadsset{n}{m} \to \realset^n \to \realset^m \]
	is inductively defined as
	\begin{alignat*}{2}
		\semtads{p}(\vec x) &\defined \semtads{p'}(\vec x) \qquad &&\text{if } p \transition{l} p' \land \sem{p}(\vec x) = l \\
		\semtads{\aff}(\vec x) &\defined \aff(\vec x) \qquad &&\text{if } \aff \not\to
	\end{alignat*}
	for a TADS $t = (N, \to, \zeta)$, with $p,p',\aff \in \nu$.
	For convenience we introduce the shorthand $\tadssem{t} \defined \tadssem{\zeta}$.
\end{definition}

Semantically, both PLNNs and TADS represent piecewise affine functions. Moreover, PLNNs can be transformed into TADS:
\begin{lemma}[Trinity: PLNNs, TADS, and PAFs]
	There exists a semantics preserving transformation
	\[ \theta \colon \netset \to \tadsset{}{} \]
	from PLNNs to TADS, such that the following diagram commutes:
	\begin{center}
		\begin{tikzcd}
			 \netset\ar[r,"\theta"] \ar[dr,swap,"\netsem{}"] & \tadsset{}{} \ar[d, "{\tadssem{}}"] \\
			& \pafset 
		\end{tikzcd}
	\end{center}
\end{lemma}

\highlight{TADS are computationally transparent and semantically equivalent to PLNNs.}

\paragraph*{Algebraic Properties. }\ Much like ADDs and BDDs, TADS inherit the algebraic properties of their leaf algebra.
For TADS, the leave algebra---affine functions---forms a vector space.
Using \emph{lifting} one can directly implement the vector space operations on TADS \cite{schlueter2022towards}.
\begin{lemma}[Lifting]
	Lifting addition ($+$) and scalar multiplication ($\cdot$) from affine functions to TADS gives semantically equivalent operators to their PAF counterparts, i.e., for all TADS $t_1,t_2\in\tadsset{}{}$
	\begin{align*}
		 \semtads{t_1 + t_2} &= \semtads{t_1} + \semtads{t_2} \\ 
		 \semtads{s \cdot t} &= s \cdot \semtads{t} \eqeos
	\end{align*}
\end{lemma}
By the lifting theorem of~\cite{schlueter2022towards} the algebraic properties are preserved and thus:
\begin{theorem}[TADS vector space]
	TADS $(\tadsset{}{}, +, \cdot)$ form a vector space.
\end{theorem}
It is well-known that piece-wise affine functions are closed under composition.
Even though this operator can not be directly lifted, it can be easily  implemented on TADS \cite{schlueter2022towards}.
\begin{theorem}[TADS Composition]
	TADS composition
	\[\Join : \tadsset{n}{m} \times \tadsset{m}{r} \rightarrow \tadsset{n}{r}\]
	is defined such that for all $t_1 \in \tadsset{n}{m}$, $t_2 \in \tadsset{m}{r}$:
	\[ \tadssem{t_1 \Join t_2} = \tadssem{t_1} \circ \tadssem{t_2} \eqeos \]
\end{theorem}
It follows straightforwardly that:
\begin{theorem}[TADS Monoid]
	TADS $(\tadsset{}{}, \Join)$ forms a monoid.
\end{theorem}
The composition operator $\Join$ is especially important in the context of neural networks, as neural networks are inherently compositions of piece-wise affine functions.

\subsection{Principal Component Analysis}
\label{subsec:PCA}
Principal Component Analysis (PCA) is one of the most popular techniques for dimensionality reduction and feature extraction \cite{wold1987principal,abdi2010principal,bro2014principal}.
At a high level, it seeks to find, for a given dataset $D\subset \realset^n$, a linear subspace $V \subset \realset^n$ with dimension
 $\dim(V)\ll n$ that can be used to encode $D$ with as little reconstruction loss as possible.

\begin{figure*}[ht]
	\centering
	\subfloat[A PCA analysis on a set of normally distributed sample inputs $X$. The green line indicates the axis represented by the first principal component. Note that the vectors scatter mostly along this axis.]
	{
	\includegraphics[width=.31\linewidth]{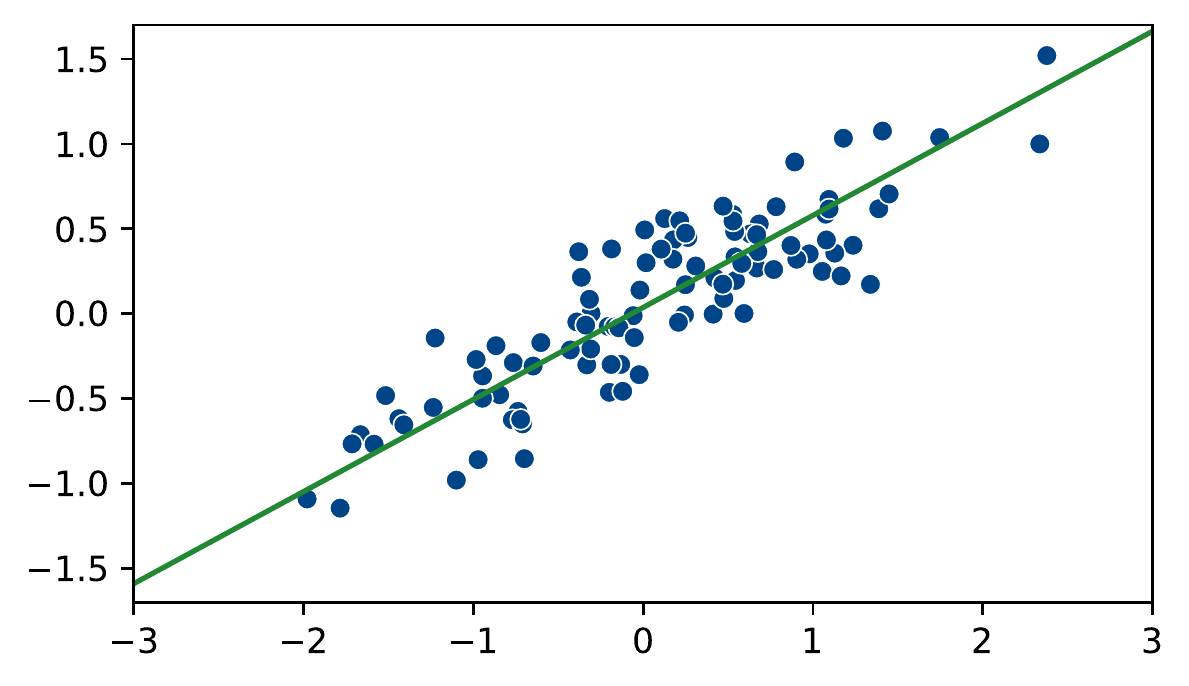}
	\label{fig:pca_graph}
	} \hfil
	\subfloat[PCA dimensionality reduction from $\realset^2$ to $\realset^1$. The PCA representation is given by $\rho_1(X)$.]
	{
		\includegraphics[width=.31\linewidth]{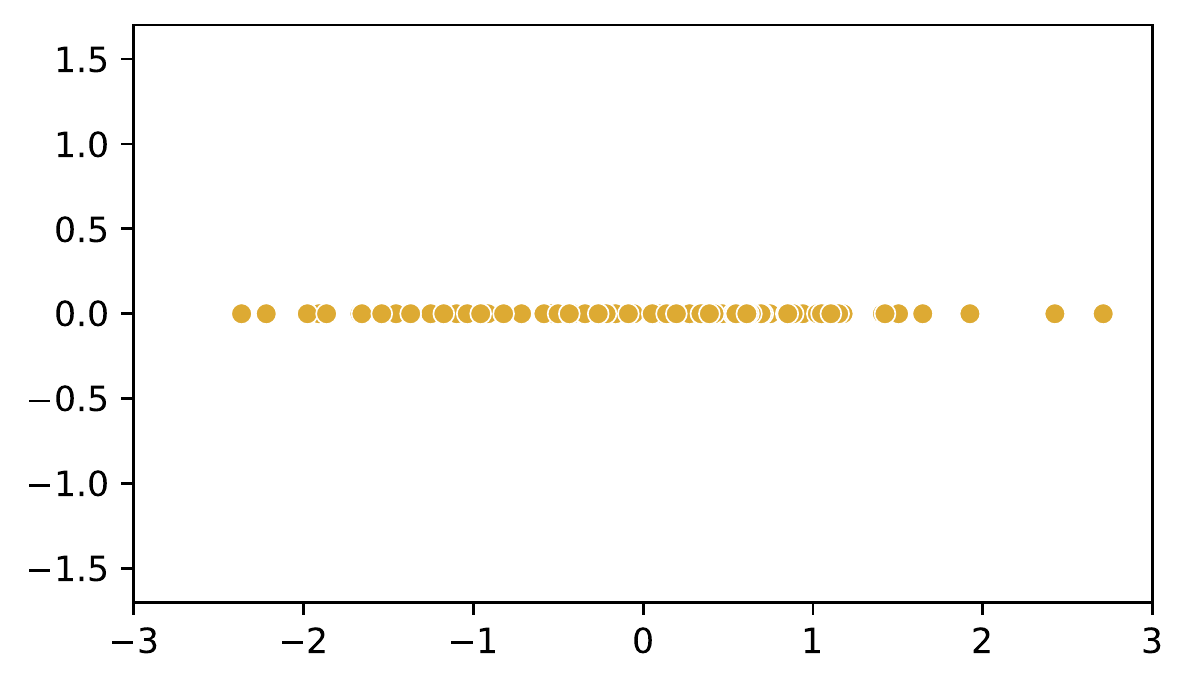}
		\label{fig:pca_representations}
	} \hfil
	\subfloat[The reconstructions $\theta_1(\rho_1(X))$ (yellow) of the original points $X$ (blue). Note that the reconstruction is very faithful, with reconstructed points being very close to their original coutnerparts.]
	{
		\includegraphics[width=.31\linewidth]{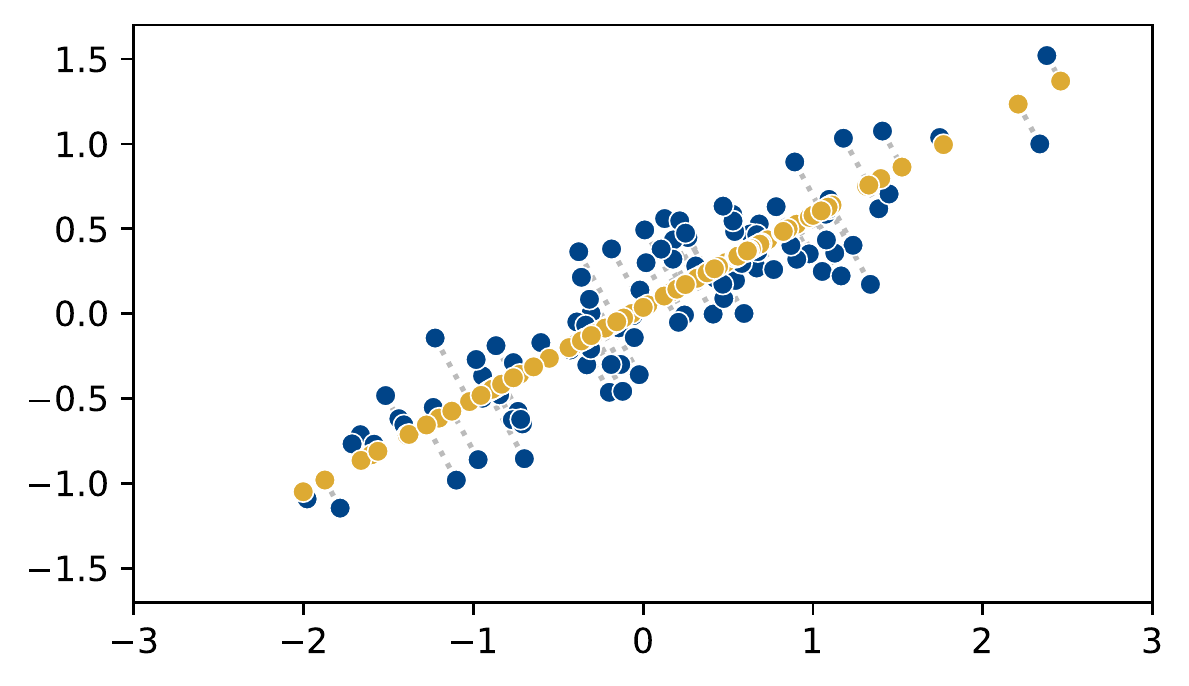}
		\label{fig:pca_reconstruction}
	}
	\caption{Example for PCA dimensionality reduction. A set of points $X$ (shown in blue) follows a multinormal distribution that scatters more along one axis. This axis is very closely resembled by the first principal component (shown in green). Through orthogonal projection one can reduce the dimensionality (b), and the reconstruction (c) is very close to the original. }
\end{figure*}

Such an encoding is useful for machine learning algorithms as it can drastically reduce the input dimension. Large input dimensions can be very problematic in machine learning and entail numerous potential problems, altogether known as the \textit{curse of dimensionality} \cite{theodoridis2006pattern}. 

The fundamental objects of PCA are the eponymous principal components that are defined as follows:
\begin{definition}[Principal Components]\label{def:PCA}%
For a given dataset ${D=\{\vec d_1,\dots, \vec d_j\} \subset \realset^n}$ with ${j \geq n}$ and zero mean $\sum_{\vec d \in D} d = \vec 0$, there exist $n$ \emph{principal components} $\vec{p_1},\dots, \vec{p_n} \in \realset^n $ which are characterized as iterative solutions to the following optimization problem: The $i$-th principal component $\vec{p_i}$ maximizes the variance of the data when it is projected onto $\vec{p_i}$:
\[ \sum_{d \in D} \innerprod{\vec{p_i}}{\vec d}^2 \rightarrow \max \] 
under the constraint that $p_i$ has unit length
\[ \norm{\vec p_i}_2 =1 \]
and is orthogonal to all previous principal components
\[ \forall h < i : \innerprod{\vec{p_i}}{\vec{p_h}} = 0 \eqeos \]
\end{definition}
Note that every dataset $D$ with non-zero mean, i.e., $\sum_{\vec d \in D} \frac{\vec d}{|D|} = \vec \mu$ with $\vec \mu \neq 0$, can be made to obey the restriction $\vec \mu =\vec 0$ by performing the following transformation on each datapoint: $\vec d_i' = \vec d_i - \vec \mu$.
%
%
By definition, the principal components are pair-wise orthogonal and normed and therefore linearly independent. Thus, they form a basis of $\realset^n$. 
It follows that there is a natural, unique representation based on the principal components $\rho(\vec x)=(r_1,\dots, r_k)^\transpose$ such that:
\[ \vec x =  \sum_{i=1}^n r_i \vec p_i \eqeos \]
In particular, in the case of PCA, the $r_i$ can be computed as 
\[ r_i = \innerprod{\vec p_i}{\vec x} \eqeos \]
With this, PCA can naturally be used as a dimensionality reduction tool.
\begin{definition}[PCA Dimensionality Reduction]%
\label{def:PCAReduction}%
\begin{figure*}[t]
	\centering
	\subfloat[Original MNIST digits, ascending from 0 to 9.]{%
		\includegraphics[width=0.8\linewidth]{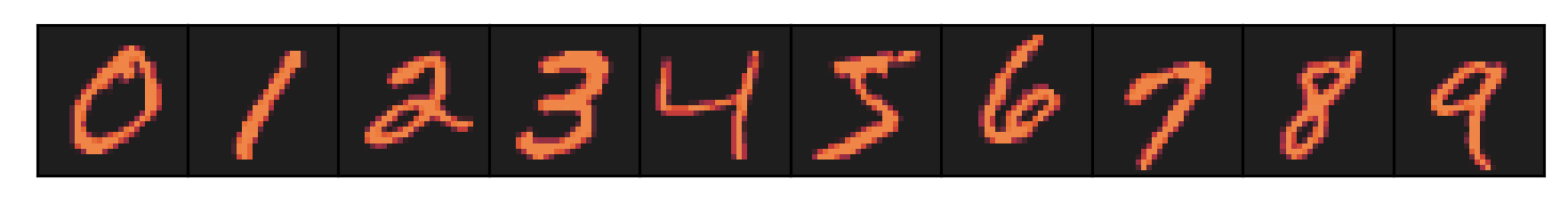}
	} \\[-2mm]
	\subfloat[PCA reconstruction with $k=2$.]{%
		\includegraphics[width=0.8\linewidth]{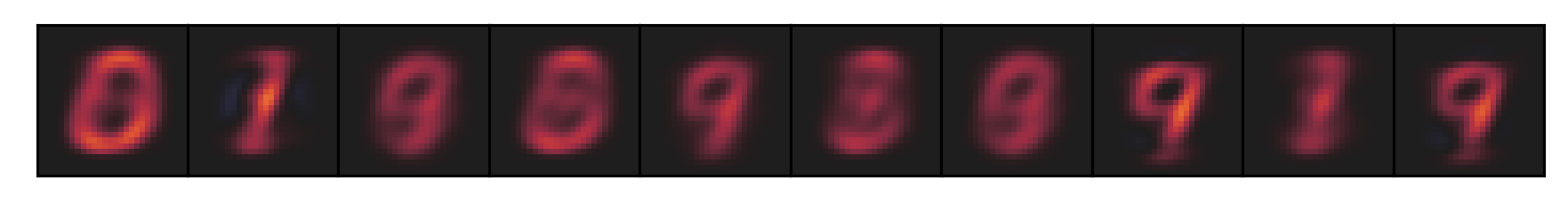}
	} \\[-2mm]
	\subfloat[PCA reconstruction with $k=8$.]{%
		\includegraphics[width=0.8\linewidth]{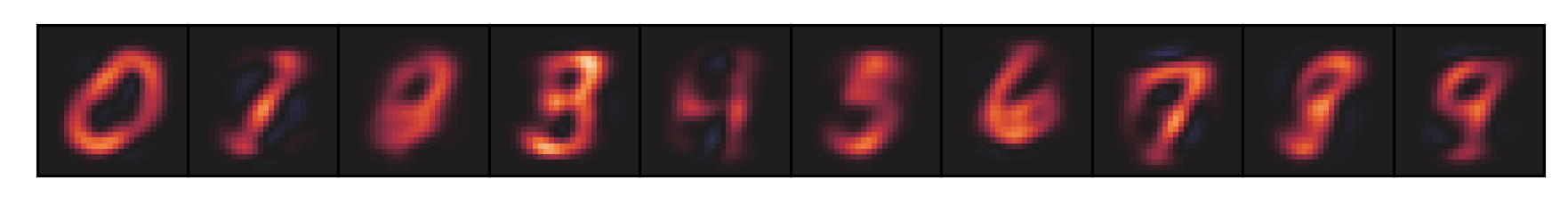}
	} \\[-2mm]
	\subfloat[PCA reconstruction with $k=16$.]{%
		\includegraphics[width=0.8\linewidth]{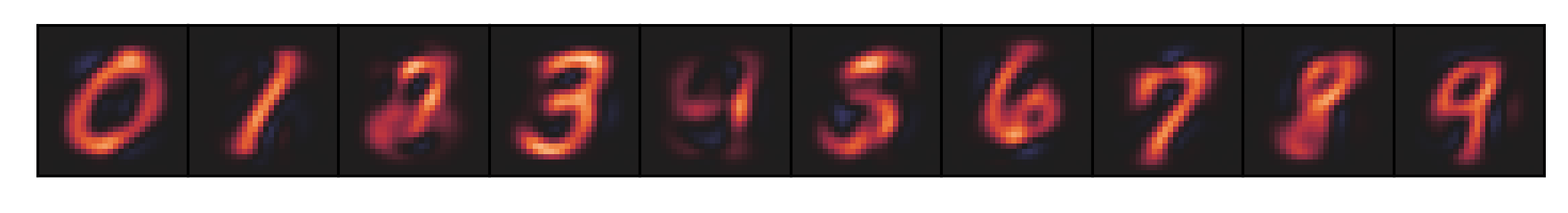}
	}
	\caption{Reconstruction of $784$ pixel MNIST digits with various number of principal components, showcasing how a PCA reconstruction can faithfully reconstruct an original input based on very little information ($k=16$ principal components vs. $784$ pixel). Original MNIST digits in first row, following rows show reconstruction with $j=2, 8, 16$ principal components.
	Vectors are visualized using a perceptually uniform diverging color palette (Seaborn's ``icefire'').}
\end{figure*}
Let $0<k< n$. For some $\vec x \in \realset^n$ with $\rho(\vec x)=(r_1,\dots, r_n)$, the \textit{$k$-dimensional PCA representation} is given by cutting off the PCA representation after the $k$-th element:
\[ \rho_k(\vec x) = (r_1,\dots, r_k)^\transpose \eqeos \]
Consequently, the \textit{$k$-dimensional PCA reconstruction} to $\vec x$ is given as:
\[ \vec x \approx \theta_k(\rho_k(\vec x))= \sum_{i=1}^k r_i \vec p_i \eqeos \]
\end{definition}
As $\rho_k$ is a projection for $k < n$, it loses information. Therefore, the PCA reconstruction after dimensionality reduction is approximative.

In essence, the composition of PCA encoding and reconstruction forms a function that is close to the identity function \textit{on the dataset and its surrounding points} while reducing the number of dimensions needed to express the data. The success of PCA is heavily dependent on the dataset being mainly distributed along a linear subspace of $\realset^n$ and its generalization performance requires that new data follow the same distribution as the training data. 
However, if these assumptions hold, it is a very good approximation, as indicated by the following defining property of the principal components:
\begin{lemma}
	The principal components are exactly those vectors that make the reconstruction error minimal over $D$ among all linear, orthogonal encoders and decoders using $k$ dimensions \cite{abdi2010principal}. That is, for all orthogonal, linear functions $e : \realset^n\rightarrow \realset^k$, and $d : \realset^k \rightarrow \realset^n$, the term
	\[ \sum_{\vec x \in D} \norm[big]{\vec x - d(e(\vec x))}_2^2 \]
	is minimal iff $e = \rho_k$, $d=\theta_k$.
\end{lemma}
PCA is attractive for multiple reasons. First, PCA representations and approximations are linear functions which makes them easy to work with. Second, PCA supports reductions to $k$ for any $0<k\leq n$, which makes PCA very flexible. Lastly, but perhaps most importantly, PCA is a well-understood and well-proven method in practice and can elegantly enable strong performance in even relatively simple machine learning models.
An example of a PCA encoding and reconstruction is shown in \cref{fig:pca_graph,fig:pca_representations,fig:pca_reconstruction}.

\highlight{PCA allows the compression of high-dimensional inputs into lower-dimensional representations such that a given dataset is compressed with as little information loss as is possible using orthogonal, linear compression.}

\section{Problem Setting: Robustness on MNIST}
\subsection{Introduction to MNIST}
\label{sec:Setting}
\label{subsec:Basic Case}

In the remainder of the paper we consider the problem of digit recognition using the MNIST dataset \cite{deng2012mnist}. 
The MNIST dataset provides a traditional baseline-problem scenario for machine learning. While simpler than modern, large-scale machine learning tasks, MNIST requires PLNNs of relevant size for satisfactory classification 
 and stands to this day as an introductory problem in verification benchmarks \cite{bak2021second}.

The MNIST dataset consists of 70.000 gray-scale images of hand-written digits, each labeled with the digit they represent to a human observer. The dataset is split into 60.000 examples for training and 10.000 examples for testing. Images consist of $28\times 28$ pixels and are represented as vectors $\vec x \in \realset^{28\cdot 28}$ with each component $\vec{x}_i$ representing the gray-scale value of the $i$-th pixel on a scale from $0$ to $1$. Thus, each sample has the form
\[ (\vec x, l) \]
with $\vec x \in [0,1]^{28\cdot 28}$ and $l \in \{0,\ldots,9\}$.
The task is to find a PLNN classifier that represents a function 
\[ \classifier\colon \realset^{28*28} \rightarrow \{0,\ldots,9\} \] assigning to each image the digit it is supposed to represent. 
At a baseline, $\classifier$ should classify most training examples correctly and should perform acceptably well on the test data.

A challenge for \emph{classification problems} like this is to control so-called \emph{adversarial examples} as discussed in the following.
\begin{figure*}
	\centering
	\subfloat
	[We consider robustness around the point $\vec x= (2,3,3)^\transpose$.
	The neighborhood defined by the $\infty$-ball (blue cube) lies on the same side of the classification boundary, thus the point is robust.]
	{
		\includegraphics[width=.48\linewidth]{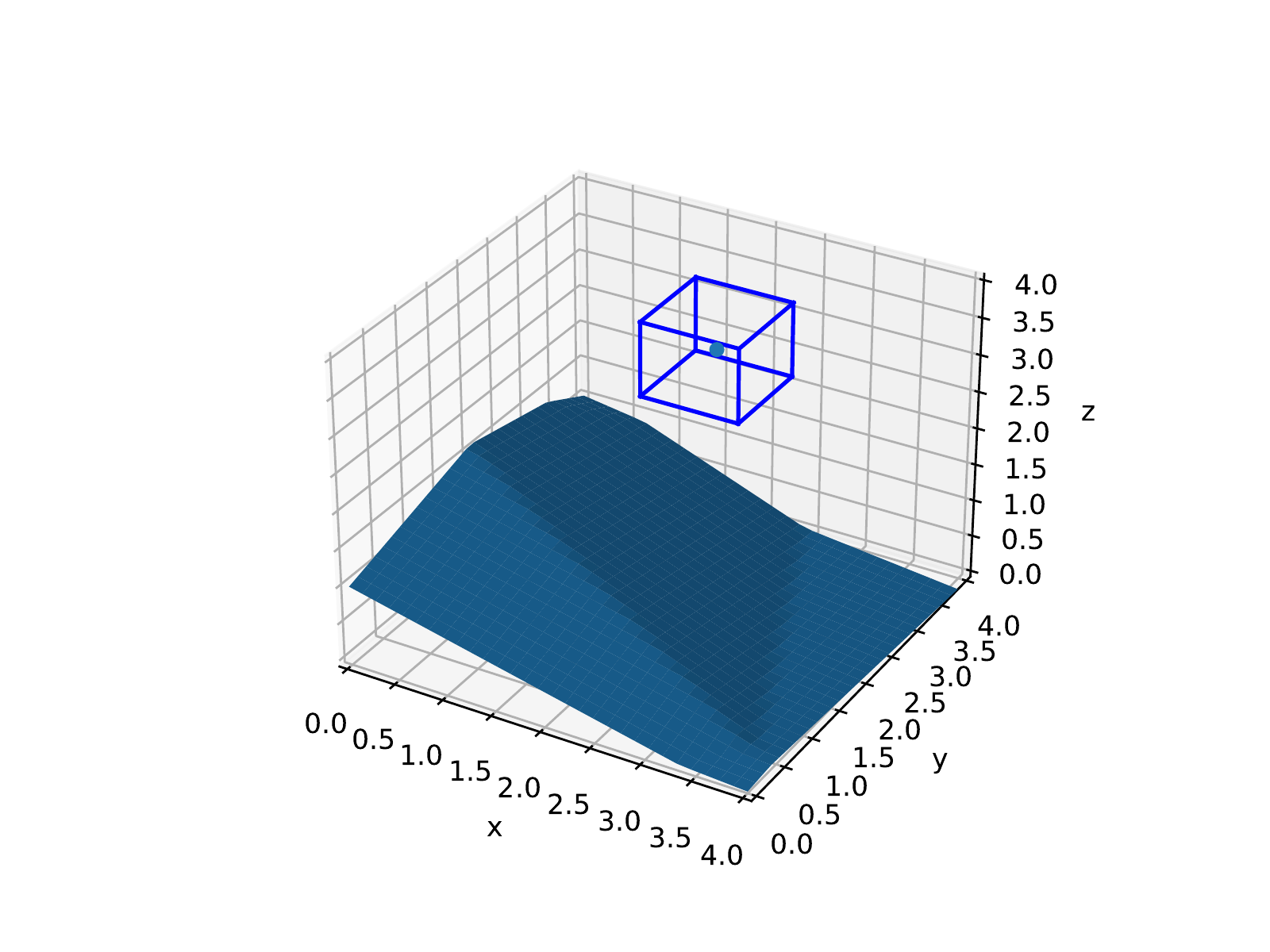}
		\label{fig:robustnessballexample}
	} \hfil
	\subfloat
	[Using the same PLNN, robustness is considered for a point close to the decision boundary.
	The $\infty$-ball intersects the boundary, thus the network is not robust for this point.]
	{
		\includegraphics[width=.48\linewidth]{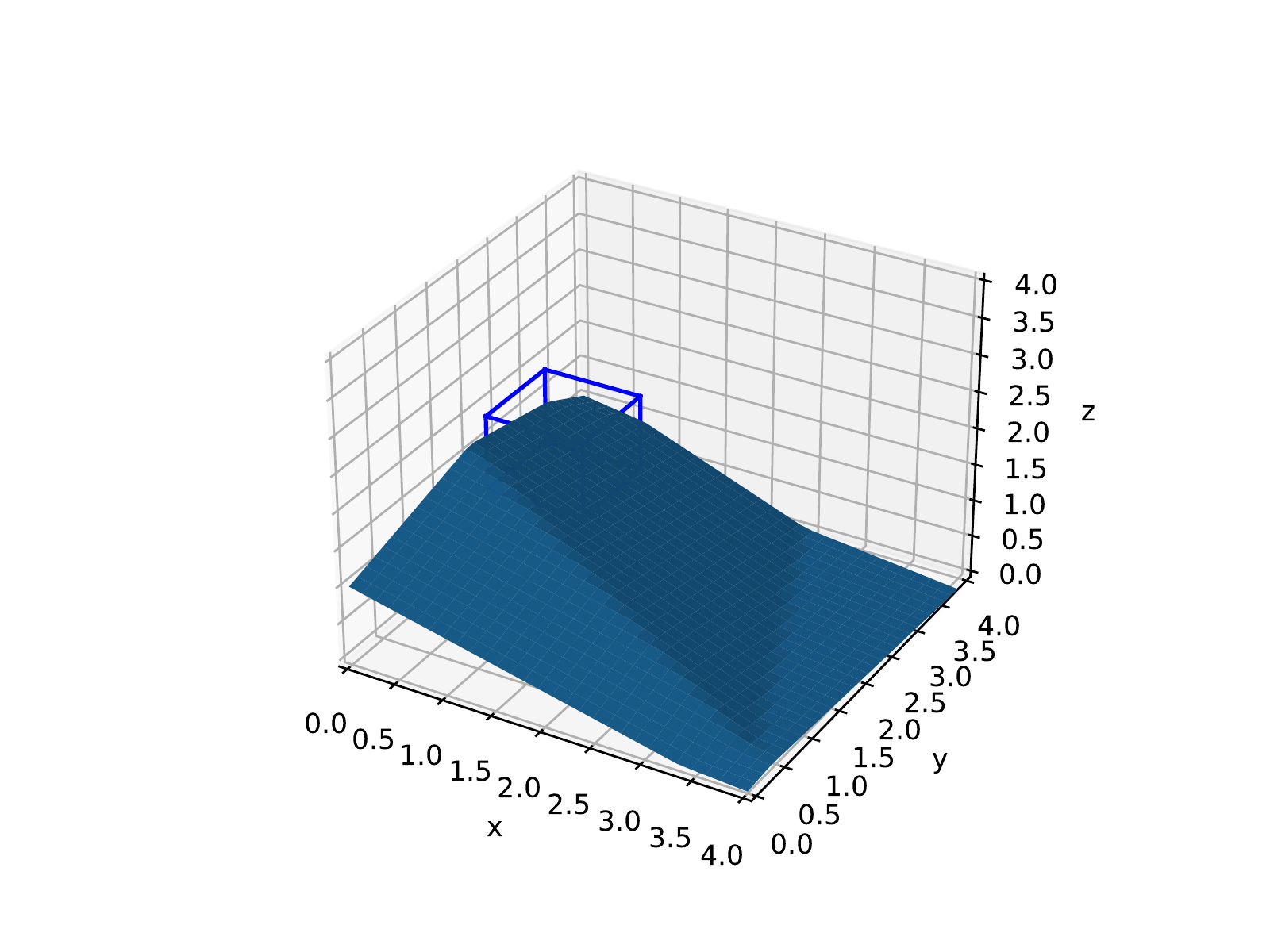}
		\label{fig:nonrobustnessballexample}
	}
	\caption{Illustration of two robustness scenarios using its geometric interpretation. In (a) robustness is achieved while in (b) robustness is violated as the $\infty$-ball around $\vec x$ intersects with the decision boundary.}
\end{figure*}

\subsection{Robustness to Adversarial Examples}
In essence, robustness is the absence of adversarial examples, which are perhaps the most well-known manifestations of chaotic behavior of neural networks and have received wide attention in research \cite{szegedy2013intriguing,goodfellow2014explaining,kurakin2016adversarial}.
We work with the following definition of adversarial examples:
\begin{definition}[Adversarial Example]
\label{def:adversarial_example}
Let $\classifier \colon \realset^n \rightarrow \{1,\dots,m\}$ be a PLNN classifier.
Further, let $\vec x \in \realset^n$ be a given point of interest that is correctly classified by $\classifier$. Then, $\vec y \in \realset^n$ is an $\epsilon$-adversarial example to $\vec x$ iff
\[ \norm{\vec y - \vec x}_\infty \leq \epsilon \: \land \: \classifier(\vec y) \neq \classifier(\vec x) \eqeos \]
If $\classifier$ admits no $\epsilon$-adversarial examples for a given input $\vec x$, then it is called \textit{$\epsilon$-robust} around $\vec x$.
\end{definition}
Intuitively, an adversarial example is a slight perturbation of an input that, although minor, changes the neural networks prediction. 
Note that in image recognition problems such as MNIST, the restriction $\inftynorm{\vec y - \vec x} \leq \epsilon$ encodes that between $\vec x$ and $\vec y$, each pixel can only differ by at most $\epsilon$.

In practice, adversarial examples can be almost imperceptible to a human \cite{luo2018towards,szegedy2013intriguing} while arbitrarily altering previously correct decisions, sometimes yielding outlandish classification results, which may 
enable outside attacks on neural network systems. 
Thus, it is critical that neural networks cannot be adversarially attacked at points where the desired semantics is clear\footnote{Of course, neural networks necessarily have regions where the prediction flips from one class to another. Ideally, these flips should only occur in regions where inputs are non-sensical and would not intuitively be assigned to any class by a human. Therefore, robustness is usually considered only at some select sample inputs where semantics is clear.}.  

\subsection{Verifying Robustness}
Generally, PLNN verification is the task proving a property for the result of a PLNN where the input is restricted to a given domain \cite{bunel2018unified,albarghouthi2021introduction}.
Formally, let ${\classifier\colon\realset^n\to\realset^m}$ be a PLNN, ${S \subset \realset^n}$ a restriction of the input domain, and ${P\colon\realset^m\to\{0,1\}}$ a predicate.
Then PLNN verification is the task of proving or refuting with a counterexample the formula
\begin{equation}\label{eq:general_plnn_verification}
	\forall \vec x \in S : P(\classifier(\vec x)) \eqeos
\end{equation}
For the case of verifying $\epsilon$-robustness around $\vec x$ for $\classifier$, we can formulate \eqref{eq:general_plnn_verification} specifically as (cf.\ \cref{def:adversarial_example})
\[ \forall \vec y \in \vec x + \epsilon B_\infty : \classifier(\vec y) = \classifier(\vec x) \eqeos \]
Corresponding state-of-the-art verification tools use different methods like \cite{bunel2018unified}:
\begin{itemize}
	\item Satisfiability Modulo Theories
	\item Mixed Integer Programming
	\item Branch and Bound
\end{itemize}
For more information, see \cref{sec:related work} on related work.

\section{Extending TADS to Cover Robustness Properties}
\label{sec:extending}

TADS are characterized by:
\begin{enumerate}
	\item \emph{Global explanations}, i.e., they explain the behavior of a PLNN over the entire space of possible inputs. Robustness properties however concern only the relatively small neighborhood $x + \epsilon B_\infty$ of a point $\vec x$.
	\item \emph{Regression behavior}, they represent a continuous function. With respect to robustness, we are however interested in the behavior of the associated \emph{PLNN classifier}. 
\end{enumerate} 
The following two subsections will show that TADS are nevertheless well-suited to deal with robustness properties.

\subsection{Precondition Projection on TADS}

When studying adversarial examples, one may use the strict preconditions (given as infinity balls $\epsilon B_\infty$) to reduce the work load.
Given the strong connection between affine functions and (convex) polytopes, it is a straightforward procedure to apply polyhedral preconditions---such as infinity balls as particularly required for robustness properties---on TADS\@.
Please note that stronger preconditions result in less work.

Given a TADS $t$ representing a piece-wise affine function 
${f \colon \realset^n \rightarrow \realset^m}$
we are interested in the behavior of $f$ on a given (small) 
polyhedron $S \subset \realset^n$. In other words, we are interested in 
the function  $\restr{f}{S} \colon S \rightarrow \realset^m$ 
which is given by:
\[ \restr{f}{S}(\vec x) =
\begin{cases}
	f(\vec x) \quad &\text{if } x \in S \\
	\bot &\text{otherwise}
\end{cases} \eqeos \]
Technically, this is implemented by encoding the polytope $S$ as a TADS using affine inequalities:
\[ \restr{t}{S} \defined \big(\restr{\mathrm{id}}{S}\big) \tadsjoin t \eqeos \]
By explicitly eliminating paths that lead to $\bot$ (see \cite{schlueter2022towards}), the resulting TADS is significantly reduced in size.
\subsection{Argmax on TADS for Classification}
\begin{figure}
	\centering
	\includegraphics[width=.7\columnwidth]{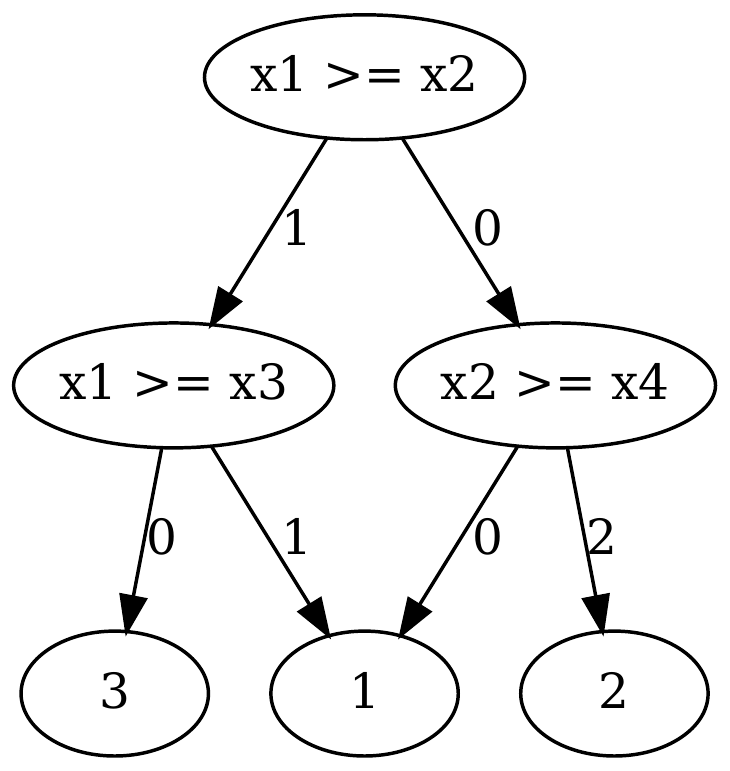}
	\caption{The TADS $t_a$ representing the argmax function with 3 variables.}
	\label{fig:argmaxtads}
\end{figure}
Neural networks are frequently used for classification, as outlined in \cref{sec:nnclassification}. As described there, the neural network classifier associated with a given neural network $\nu$ can naturally be modeled as 
\[ \nu_c = {\argmax} \circ {\sem{\nu}} \]
Interpreting a neural networks behavior in this way drastically changes its nature, and if one seeks to analyze a neural network that is meant to be used as a neural network classifier, it is important that one analyzes it with respect to its classification behavior. 

We know how to construct a $t_{\nu}$ for any PLNN  $\nu$. On the other hand, it is also easy to see how a TADS $t_a$
can be constructed for  ${\argmax}$: 
Intuitively, such a TADS needs only to perform a linear search for the maximum of ${\vec x=(x_1,\dots,x_n)}$ from $x_1$ to $x_n$. \Cref{fig:argmaxtads} illustrates this for the three-dimensional argmax.\footnote{Note that this TADS deviates slightly from the  representation of TADS we use for the rest of this paper, notably with respect to the way linear inequalities are represented. This is purely done to enhance readability.} This TADS first compares $x_1$ and $x_2$ in the first layer, then compares their maximum with $x_3$ to attain the result. The extension to higher dimensions is straightforward.

Taken together, it is straightforward to construct the classification TADS $t_{\nu_c}$ using TADS composition as follows:
\[ t_{\nu_c} = t_{\nu} \tadsjoin t_{a} \eqeos \]
The semantical correctness of this construction follows directly from the correctness of the TADS composition, i.e.:
\[ \tadssem{t_{\nu_c}} = {\argmax} \circ {\sem{\nu}} \eqeos \]

\section{Verifying Robustness on MNIST Using TADS}
\begin{figure*}
	\definecolor{olivegreen}{HTML}{BBCC33}
	\definecolor{mintgreen}{HTML}{44BB77}
	\definecolor{lightgreen}{HTML}{CCDDAA}
	\definecolor{orangebar}{HTML}{EE8866}
	\definecolor{lightblue}{HTML}{77AADD}
	\centering
	\subfloat[Standard Verification Task]{%
		\includegraphics[width=0.22\linewidth]{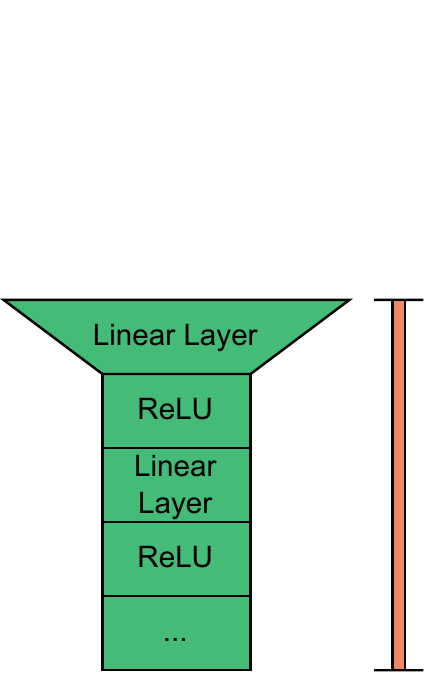}
		\label{fig:pca_plnn_base_case}
	} \hfil
	\subfloat[PCA Search Space]{%
		\includegraphics[width=0.22\linewidth]{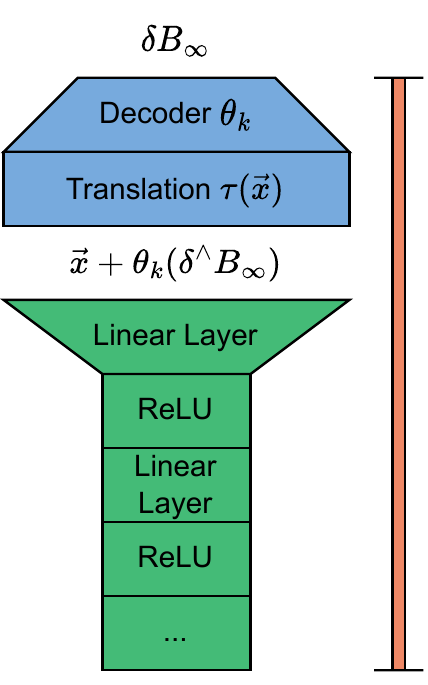}
		\label{fig:pca_plnn_heuristic}
	} \hfil
	\subfloat[End-to-end Verification]{%
		\includegraphics[width=0.22\linewidth]{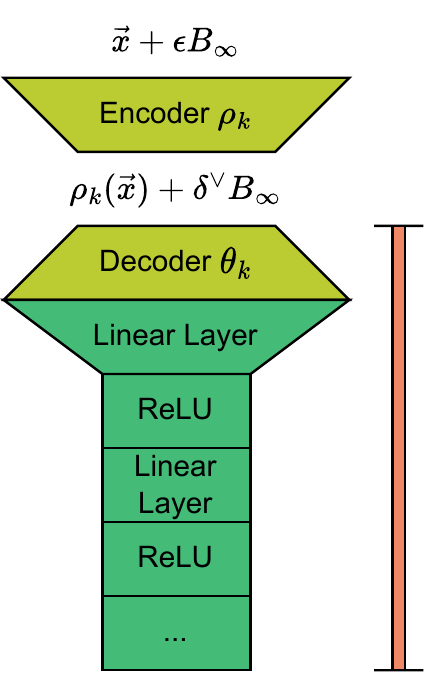}
		\label{fig:pca_plnn_complete}
	} \hfil
	\subfloat[Reduced Parameter Space]{%
		\includegraphics[width=0.22\linewidth]{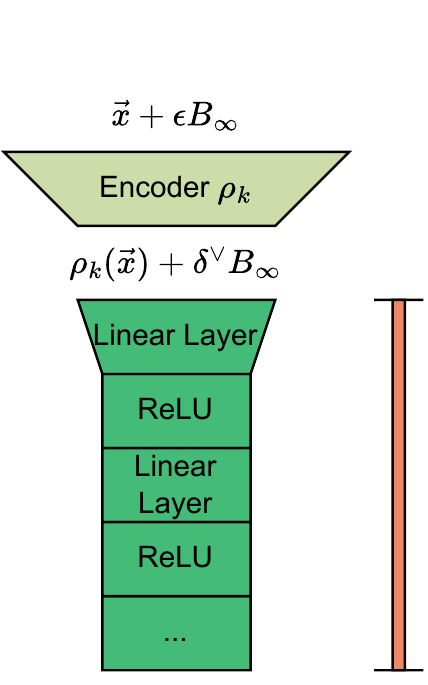}
		\label{fig:pca_plnn_learned}
	}
	\caption{Overview of the different approaches to verifying robustness with PCA encoding.
	Legend: (green) entity for usage in the real world, (\textcolor{lightblue}{blue}) components only used during verification, (\textcolor{orangebar}{orange sidebar}) components used for TADS construction, (\textcolor{mintgreen}{mint green}) components used during training and actively trained, (\textcolor{lightgreen}{pale green}) parts included in training but not actively learned, (\textcolor{olivegreen}{olive green}) parts that are not included in the training process.}
	\label{fig:methods_overview}
\end{figure*}
The following subsections of this Section present four approaches to robustness verification via TADS and illustrate them using the MNIST data set:
\begin{itemize}
	\item
	A straightforward approach where the considered PLNN is directly transformed into a TADS (cf., \cref{fig:pca_plnn_base_case}).
	This approach typically does not scale due to the typical exponential explosion of the TADS transformation.
	\item 
	An approximative approach based on PCA-based dimensionality reduction that scales, provides a good heuristics to search for 
	adversarial examples, but is insufficient to prove robustness (cf., \cref{fig:pca_plnn_heuristic}). In this case, the TADS-based analysis only 
	covers the subspace that can be `reached' from the initial, low-dimensional PCA-based vectors space via 
	decoding and adequate basis transformation as indicated by the blue part. Thus, this approach cannot guarantee that the analysis of the TADS is sufficient 
	to reveal all adversarial examples of the original PLNN.
	\item
	A transformational approach based on PCA-based dimensionality reduction, where the PLNN is extended by a preprocessing step,
	defined by PCA-based auto-encoding, i.e., the composition of a PCA-based dimensionality reduction followed by a linear function 
	that embeds (decodes) the low-dimensional space into the original space (cf., \cref{fig:pca_plnn_complete}). Here we can show that analyses of the partial extension
	that start with the decoding are sufficient to obtain robustness results for the extended PLNN that is defined for the 784-dimensional space
	of MNIST.
	\item
	A modification of the third approach, where the linear function defined by the composition of the 
	decoder and the initialization layer of the original net is replaced by a linear layer to provide a 
	network architecture with the same number of layers but with a strongly reduced input dimension (cf., \cref{fig:pca_plnn_learned}).
	The PLNN considered for verification is now given as the result of a learning process using the
	same sample set as in the other cases, but starting with a PCA-based reduction step. Technically,
	the subsequent TADS-based robustness analysis proceeds exactly in the same way as before 
	guaranteeing that the robustness result proven for the dark green part can again be lifted to the 
	overall net.
\end{itemize}
We will show that the third and fourth approaches allow us to prove full robustness in a computationally efficient manner. However, they come at the price of modifying the PLNN. 
In our eyes, this is no disadvantage as long as the modified PLNN is still sufficiently accurate; Neural networks are themselves only results of a heuristic training process and have no intrinsic merit beyond their predictive accuracy.
In fact, the results shown in \cref{fig:PCAaccplot} indicate that predictive accuracy can still be achieved after a significant reduction in dimensionality, drastically easing formal verification.

\subsection{Full Verification with TADS}
\label{sec:cases_tads}
At their baseline, TADS are so called \emph{model explanations} \cite{guidotti2019ai} of PLNNs, i.e., for any classification PLNN ${\nu_c \colon \realset^n \rightarrow \netclass}$, a corresponding TADS can be generated that represents the same function as $\nu_c$ in an easily comprehensible and analyzable manner.
Of course, the global behavior of neural networks is usually too large to be represented with a TADS\@.
However, in the case of robustness verification, we are only interested in the behavior of $\nu_c$ in the neighborhood around some point of interest $\vec x$, formalized by an infinity ball (see \cref{def:mball}).
Recall from \cref{def:adversarial_example} that $\epsilon$-robustness for a point $\vec x$ is formalized by the property
\[ \forall \vec y : \norm{\vec y - \vec x}_\infty \leq \epsilon \implies \nu_c(\vec x) = \nu_c(\vec y) \]
Equivalently, this problem can also be stated as
\[ \abs{\nu_c( \vec x + \epsilon B_\infty^n)} = 1  \]
that is, the neighborhood of $\vec x$ defined by the infinity ball $\epsilon B_\infty$ of dimension $n$ with radius $\epsilon$ is classified consistently as one class.
This property can be verified using the following theorem:
\begin{theorem}[TADS Verification]
	Let $\nu_c$ be a PLNN classifier, $\vec x$ a point of interest, and $t$ a TADS satisfying $\tadssem{t} = \netsem{\nu_c}$.
	Then $\nu_c$ is $\epsilon$-robust around $\vec x$ iff 
	\[ \restr{t}{\epsilon B_\infty^n}\ \]
	contains only feasible paths to the class $\nu_c(\vec x)$.
\end{theorem}
The correctness of this theorem follows directly from the correctness results regarding TADS that were established in \cref{sec:extending}.

The approach to directly verify the original network sketched at the very left of \cref{fig:methods_overview} only works 
for quite small MNIST networks. Core reason for this scaling problem is the dimensionality of MNIST: With 784-dimensional 
inputs, the volume of the $\epsilon$-ball around $\vec x$ is proportional to $\epsilon^{784}$ which grows quite quickly 
leading to intractably large TADS. 

\highlight{The complexity of robustness verification increases exponentially with the number of input dimensions. Reducing dimensionality is therefore key to mitigating scaling issues and proving robustness for larger $\epsilon$.}

\subsection{PCA Guided Validation}
\label{sec:cases_full_pca}
\begin{figure}[htbp]
	\centering
	\subfloat[Principal Component 1]{%
		\includegraphics[width=.4\columnwidth]{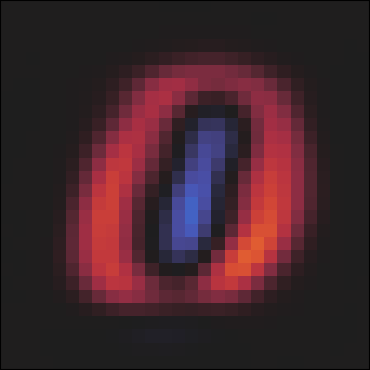}
		\label{fig:pc1}
	} \hfil
	\subfloat[Principal Component 2]{%
		\includegraphics[width=.4\columnwidth]{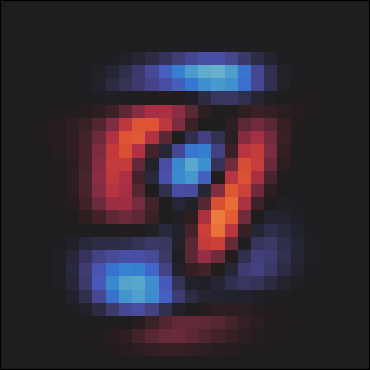}
		\label{fig:pc2}
	} \\
	\subfloat[Principal Component 3]{%
		\includegraphics[width=.4\columnwidth]{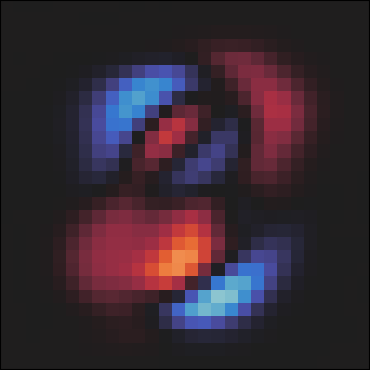}
		\label{fig:pc3}
	} \hfil
	\subfloat[Principal Component 4]{%
		\includegraphics[width=.4\columnwidth]{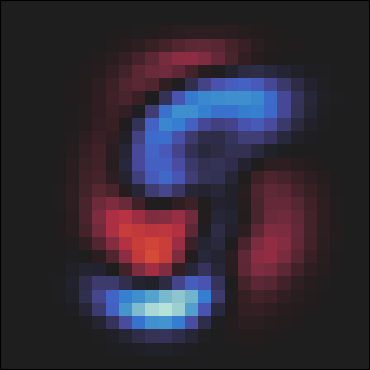}
		\label{fig:pc4}
	} \\
	\subfloat[Principal Component 5]{%
		\includegraphics[width=.4\columnwidth]{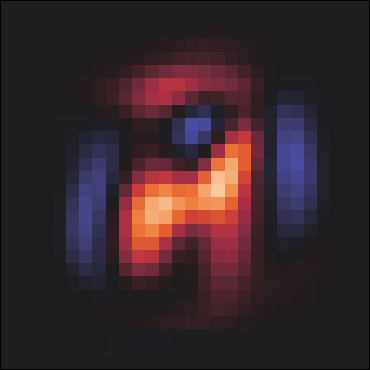}
		\label{fig:pc5}
	} \hfil
	\subfloat[Principal Component 6]{%
		\includegraphics[width=.4\columnwidth]{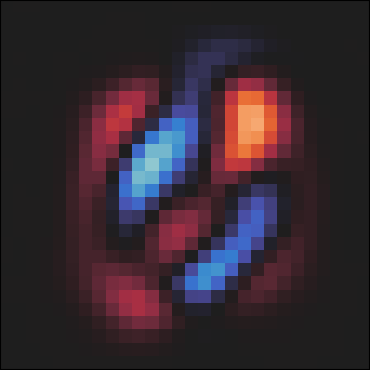}
		\label{fig:pc6}
	}
	\caption{The first 6 principal components of the MNIST dataset. Note that in the context of MNIST, images are just 784-dimensional vectors, and we therefore represent the PCA vectors as images.
	Vectors are visualized using a perceptually uniform diverging color palette (Seaborn's ``icefire'').
	Negative values are shown in blue, positives in red. Higher values are expressed with higher color intensity.}
	\label{fig:principal_components}
\end{figure}
\begin{figure}[htp]
	\includegraphics[scale=0.5]{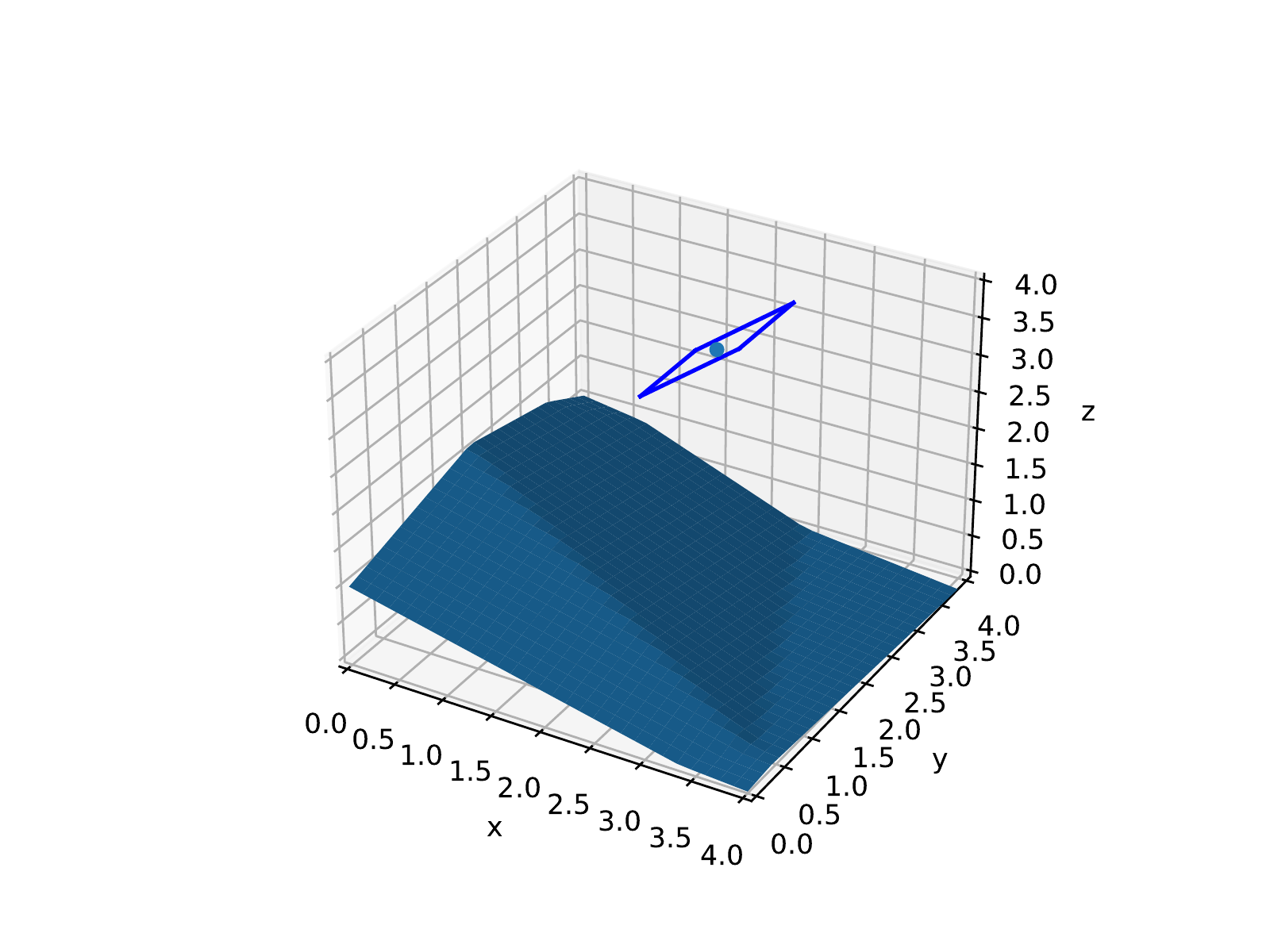}
	\caption{The same general robustness scenario that is depicted in \cref{fig:robustnessballexample}, except that now only approximative robustness is considered. }
	\label{fig:approxanalysisexample}
\end{figure}
%

To improve scalability of TADS-based verification, one might consider approximative robustness instead.
More concretely, instead of searching for adversarial examples in the full ball $\epsilon B_{\infty}(\vec x)$, we will present an approach that restricts the search to a lower-dimensional subset $S \subset \epsilon B_{\infty}(\vec x)$.

This will yield an underapproximation to robustness: If an adversarial example is found in $S$, it also exists in $B_{\infty}(\vec x)$ and robustness is violated.
However, the absence of adversarial examples in $S$ does not imply the absence of adversarial examples in $B_{\infty}(\vec x)$.
Key for the construction of the lower-dimensional manifold $S$ is \emph{principal component analysis} (PCA) as introduced in \cref{subsec:PCA}.

Applying PCA to MNIST results in a list of $n=784$ principal components 
\[ \vec{p}_1, \dots, \vec{p}_{n} \] 
ordered by decreasing variance along the respective axis.

The first six of which are visualized in \cref{fig:principal_components}.
Recall from \cref{def:PCAReduction}, that the principal components of $\vec p_i$ are precisely those along which a given dataset scatters most. They are therefore natural candidates to explore in a heuristic search for adversarial examples.

Let $\vec x$ be some point for which we seek to find adversarial examples. Then, we can define the $k$-dimensional PCA space around $\vec x$ as follows:
\[ U_k(\vec x) = \set*{\vec x +  \sum_{i=1}^k w_i \vec p_i  \given  w_i \in \realset } \eqeos \]
This space contains all vectors that are reachable from $\vec x$ along the principal components or, equivalently, the image of the PCA decoding function $\theta_k$.

This allows us to define a search space for adversarial examples:
\begin{equation}\label{eq:unit_search_space}
	S \defined  U_k(\vec x)  \  \cap \   (\vec x + \epsilon B_\infty^n) \eqeos
\end{equation}
Observe that $S$ is by definition a subset of $B_\infty^n$ of dimensionality $k \ll n$ and that, as both $U_k$ and $ \epsilon B_\infty^n$ are defined by linear equations, it can be conveniently expressed as a TADS precondition.

Restricting the search for adversarial examples to $S$ via a PCA-based transformations that adequately decodes the vectors of some k-dimensional ball $\delta B_\infty^k$, as sketched in \cref{fig:pca_plnn_heuristic}, drastically reduces the computational load.\footnote{Considerations about a good choice of $\delta$ are postponed to the next section, where we determine a $\delta$ that is suffices to prove robustness with he method indicated in \cref{fig:pca_plnn_complete}.} However, this reduction comes at a price, which is usually quite high (cf., \cref{fig:approxanalysisexample}): Independently of the choice of $\delta$, it can never prove the absence of adversarial examples.
	
\highlight{Dimensionality reduction mitigates scalability issues, but can only be considered as a heuristics for finding adversarial examples.}

\subsection{Built-in PCA Verification}
\label{subsec:pca verif}
Fundamentally, neural networks are heuristic models that seek \emph{only} to achieve high performance, which is typically defined as the accuracy of their predictions. 
If a neural network is only as useful as its predictive accuracy, then any change that is made to the neural network that does not drastically alter its predictive accuracy is acceptable.
This opens up a new angle to neural network verification that is unlike traditional program verification: Rather than trying to verify a given network as it is, one may well alter the network as long as 
this does not impair the prediction quality too much. In fact, we consider such a step (often)
necessary, as classifiers defined by high-dimensional neural networks will often not be robust, 
but small alterations may well be.

\Cref{fig:pca_plnn_complete} sketches how the idea of PCA
can be used to achieve such an alteration: The point is that each input is channeled through
the low-dimensional PCA space, which, similar to the situation in the previous section,  
is simple enough to support verification.
However we will see that, in contrast to the previous section, the special character of 
PCA encoding allows us to infer robustness result from the robustness results for
the PCA space. More concretely, after verifying the robustness of 
\[ \nu_c \circ \theta_k  \]
we establish a robustness result for the full modified net 
\[ \nu_r \defined \nu_c \circ \theta_k \circ \rho_k \eqeos \]
The success of this method very much depends on the accuracy
 of  $\nu_r$, which itself strongly depends on the chosen $k$.
We will discuss this issue in the \cref{sec:cases_accuracy}.

In the remainder of this section, we show how to infer robustness result for
the full $n$-dimensional vector space from robustness results for $\nu_c \circ \theta_k$.
Key observation to prove this property is that PCA preserves neighborhoods:

\begin{lemma}[PCA Preserves Neighborhoods]
	\label{lem:pca_robustness}%
	Let $\rho_k$ be the PCA transformation of the first $k$ principal components.
	For an input $\vec x$ and an $\epsilon$-neighbor $\vec y$ with ${\norm{\vec y - \vec x}_\infty \leq \epsilon}$ one can estimate their distance in the image of $\rho_k$ as
	\[ \norm{\rho_k(\vec y) - \rho_k(\vec x)}_\infty \leq \epsilon \max\nolimits_i \norm{\vec{p}_i}_1 \eqeos \]
\end{lemma}
\begin{proof}
	\begin{align*}
		&\phantom{{}={}} \norm{\rho_k(\vec y) - \rho_k(\vec x)}_\infty \\
		&= \norm{\rho_k(\vec y - \vec x)}_\infty 
		&\text{linearity}\\
		&= \max\nolimits_{i=1}^k \, \abs*{ \innerprod{\vec{p}_i} {\vec y- \vec x} }
		&\text{def. $\norm{}_\infty$} \\
		&= \max\nolimits_{i=1}^k \abs*{\sum\nolimits_{j=1}^k \, {(\vec{p}_i)_j} \cdot {(\vec{y}_j- \vec{x}_j)} }
		&\text{def } \innerprod{}{} \\
		&\leq \max\nolimits_{i=1}^k \sum\nolimits_{j=1}^k \, \abs{(\vec{p}_i)_j} \cdot \abs{\vec{y}_j- \vec{x}_j }
		&\text{$\triangle$ for $\abs{}$} \\
		&\leq \epsilon \: \max\nolimits_{i=1}^k \sum\nolimits_{j=1}^k \, \abs{(\vec{p}_i)_j}
		&\text{assumption} \\
		&= \epsilon \: \max\nolimits_{i=1}^k \, \norm{\vec{p}_i}_1 
		&\text{def. $\norm{}_1$}
	\end{align*}
	One can see that the bound is tight by setting 
	\[ \vec y = \vec x + \epsilon \sign(\vec{p}_i) \]
	where $\sign(\vec{p}_i)$ is the sign function applied component wise to $\vec{p}_i$. For that $\vec y$ equality holds for all steps.
\end{proof}

As all $p_i$ have unit length, it is possible to derive an upper bound for $\norm{\vec{p}_i}_1$ for every PCA\@.
It is obtained when at least one principal component $\vec{p_i}$ (with some $1 \leq i \leq k$) equals
\[ \vec{p}_i = \left(\pm\frac{1}{\sqrt{n}}, \dots, \pm\frac{1}{\sqrt{n}} \right)^\transpose \in \realset^n \eqeos \]
In that case the norm is $\norm{\vec{p}_i}_1 = \sqrt{n}$, leading to the following proposition.
\begin{corollary}
	For every $k \leq n$ and every set of principal components $\vec{p_1}, \dots, \vec{p_n}$ the PCA representation $\rho_k$ satisfies
	\[ {\norm{\vec y - \vec x}_\infty \leq \epsilon} \implies \norm{\rho_k(\vec x) - \rho_k(\vec y)}_\infty \leq \epsilon \sqrt{n} \eqeos \]
\end{corollary}
This suffices to prove the announced robustness result:

\begin{theorem}[Robustness]\label{lem:epsilon_delta_robustness}%
	Let $\nu\colon\realset^n\to\realset^m$ by a PLNN\@.
	Then, let
	\begin{align*}
		\nu'_r &\defined \nu \circ \theta_k \\
		\nu_r &\defined \nu \circ \theta_k \circ \rho_k \eqeos
	\end{align*}
	If $\nu'_r$ is $\delta$-robust around $\rho_k(\vec x)$ with
	\[ \delta = \max\nolimits_i \norm{\vec{p}_i}_1 \leq \epsilon \sqrt{n} \ , \]
	then $\nu_r$ is $\epsilon$-robust around $\vec x$.
\end{theorem}
\begin{proof}
	For a proof by contraposition, we show that if $\nu_r$ is not $\epsilon$-robust, then $\nu'_r$ is not $\delta$-robust either.
	Let $\vec z \in \realset^n$ be an adversarial example for $\nu_r$ with $\norm{\vec z - \vec x}_\infty \leq \epsilon$.
	By \cref{lem:pca_robustness} it follows that  $\norm{\rho_k(\vec z) - \rho(\vec x)}_\infty \leq \delta$.
	Therefore $\rho_k(\vec z) \in \rho_k(\vec x) + \delta B_\infty^k$.
	And since $\vec z$ is an adversarial, it follows as desired that
	\[  \nu'_r(\rho_k(\vec z)) = \nu_r(\vec z)\neq \nu_r(\vec x) = \nu'_r(\rho_k(\vec x)) \eqeos \]
\end{proof}
In other words, proving $\nu'_r$'s robustness on the $k$-dimensional PCA space with radius $\delta$ directly proves robustness for the entire construct $\nu_r$ with radius $\epsilon = \frac{\delta}{\sqrt n}$. In the case of MNIST, $n$ is equal to 784. Therefore, proving robustness of $\nu'_r$ for some radius $\delta$ implies robustness of $\nu_r$ with radius at least
\[ \delta \geq \epsilon \geq \frac{\delta}{28} \eqeos \]

\highlight{Altering the classifier via PCA enables full robustness verification via 
low-dimensional reasoning. However, as shown in \cref{fig:PCAaccplot}, this comes at the cost of accuracy, in particular for small $k$.}

\subsection{Improving Accuracy}
\label{sec:cases_accuracy}
As laid out in \cref{subsec:pca verif}, PCA can be used to modify a neural network in a manner that makes it much easier to verify at the cost of some predictive accuracy.
Fortunately, by modifying not only the neural network itself, but also its training process, some of that lost accuracy can be regained at almost no cost.
\Cref{fig:pca_plnn_learned} sketches a way how both, verification can be eased and
and accuracy for low $k$ can be  improved. Key to this
approach is the observation that in \cref{fig:pca_plnn_complete} the PCA decoder and
the first linear layer are adjacent and can therefore simply evaluate to a linear function with
k-dimensional input and an output dimension defined by the first hidden layer. Thus, rather
than just modifying the original classifier via PCA auto-encoding, one can (re-) learn the
entire green part through the PCA encoder. This results in a much smaller trained network $\nu_t$
which, in particular, is shielded from the 784 dimensions of MNIST by the PCA decoder. In fact, in our setup, 
\begin{itemize}
\item
the number of neurons in $\nu_t$ is essentially an order of magnitude smaller than the original net, and 
\item
the performance of $ \nu_t  \circ \rho_k $ is much better for small $k$, as shown in \cref{traineddracc}.
\end{itemize}

\highlight{Specifically training a neural network according to PCA encoding improves accuracy, in particular, for small 
PCA dimensions $k$.}

\begin{figure}
	\centering
	\includegraphics[width=.9\columnwidth]{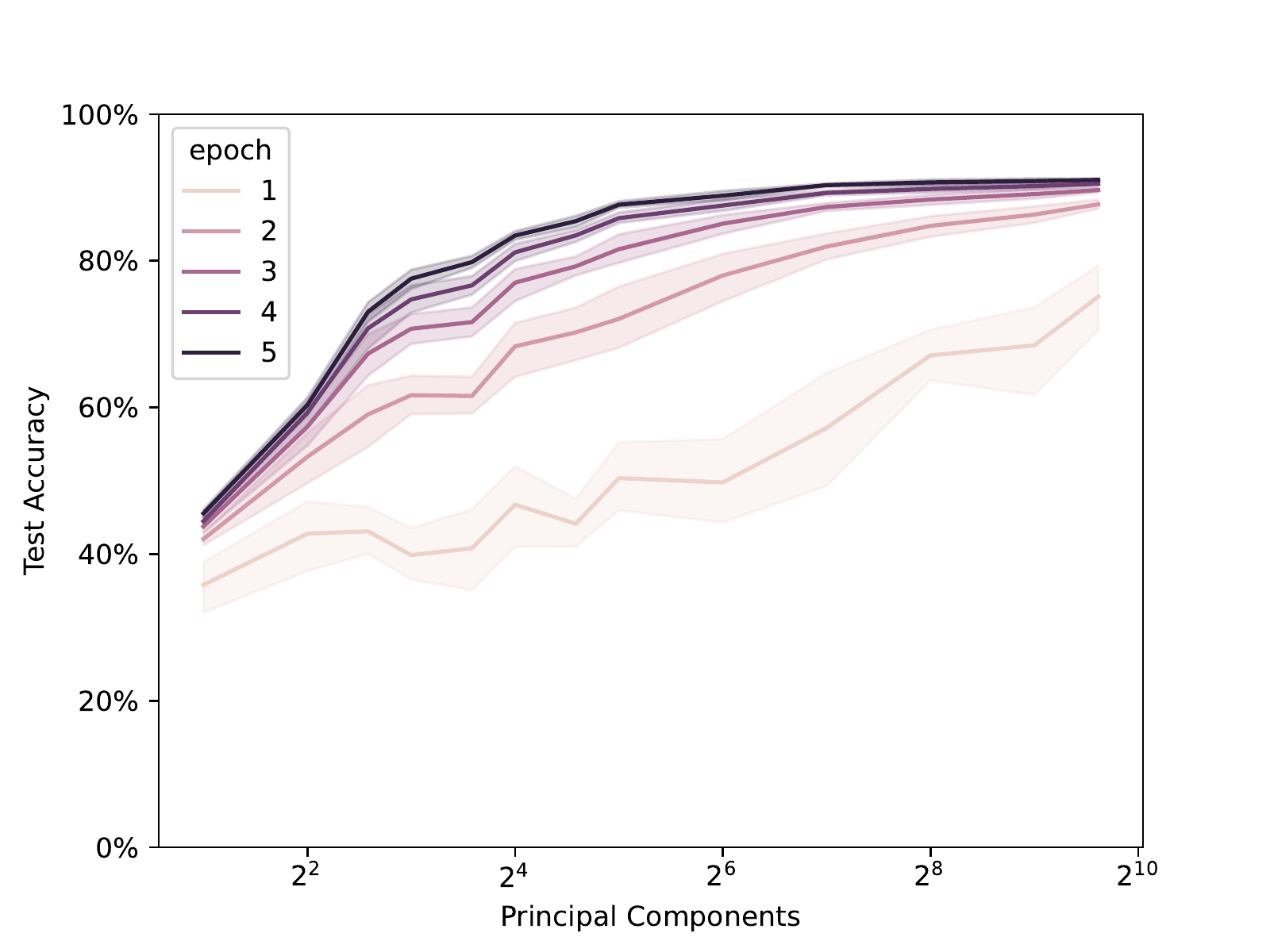}
	\caption{A plot showing the dependence of the neural network accuracy on the number $k$ of input dimensions allowed in the input encoding.
	Multiple networks were trained with different initializations for 5 epochs.
	Error bars illustrate $95\%$ confidence interval.
	Parameters: PyTorch framework with random seeds 0, 5, 10, 15, 20, 25, 42; network with 5 layers, 10 neurons per layer, ReLU activation, kaiming normal initialization, Adam optimizer, cross-entropy loss. PCA implementation of SciPy.}
	\label{fig:PCAaccplot}
\end{figure}

\begin{figure}
	\centering
	\includegraphics[width=.9\columnwidth]{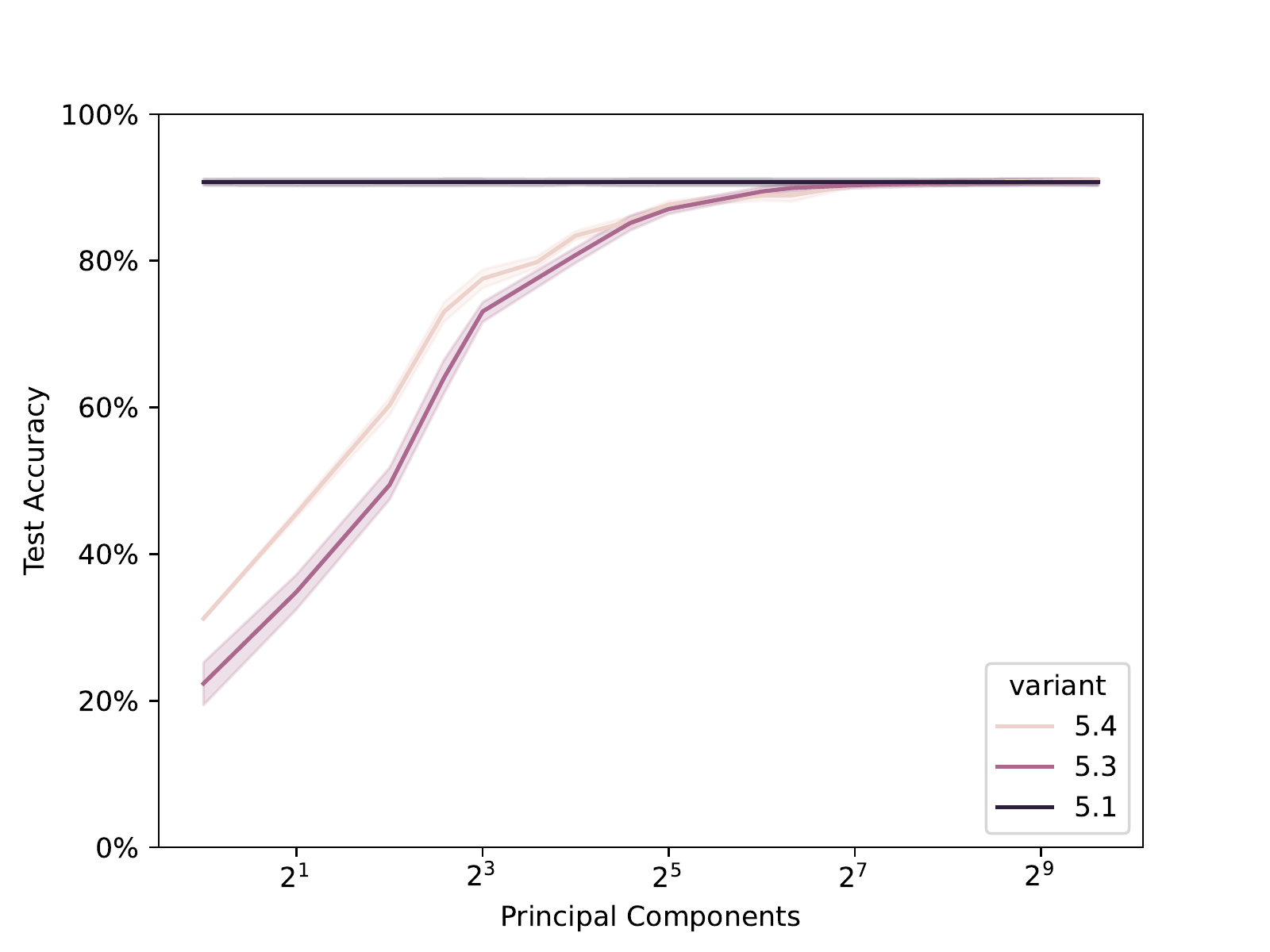}
	\caption{Comparison of predictive accuracy on MNIST's test set for the modified PLNNs of \cref{subsec:pca verif} and \cref{sec:cases_accuracy}. The violet line (``variant 5.1'') shows the reference accuracy of an unmodified PLNN with same architecture and hyperparameters. All networks were trained as in \cref{fig:PCAaccplot}, but only the accuracy after the 5-th epoch is shown.}
	\label{traineddracc}
\end{figure}

\begin{figure}[t]
	\center
	\includegraphics[width=.6\columnwidth]{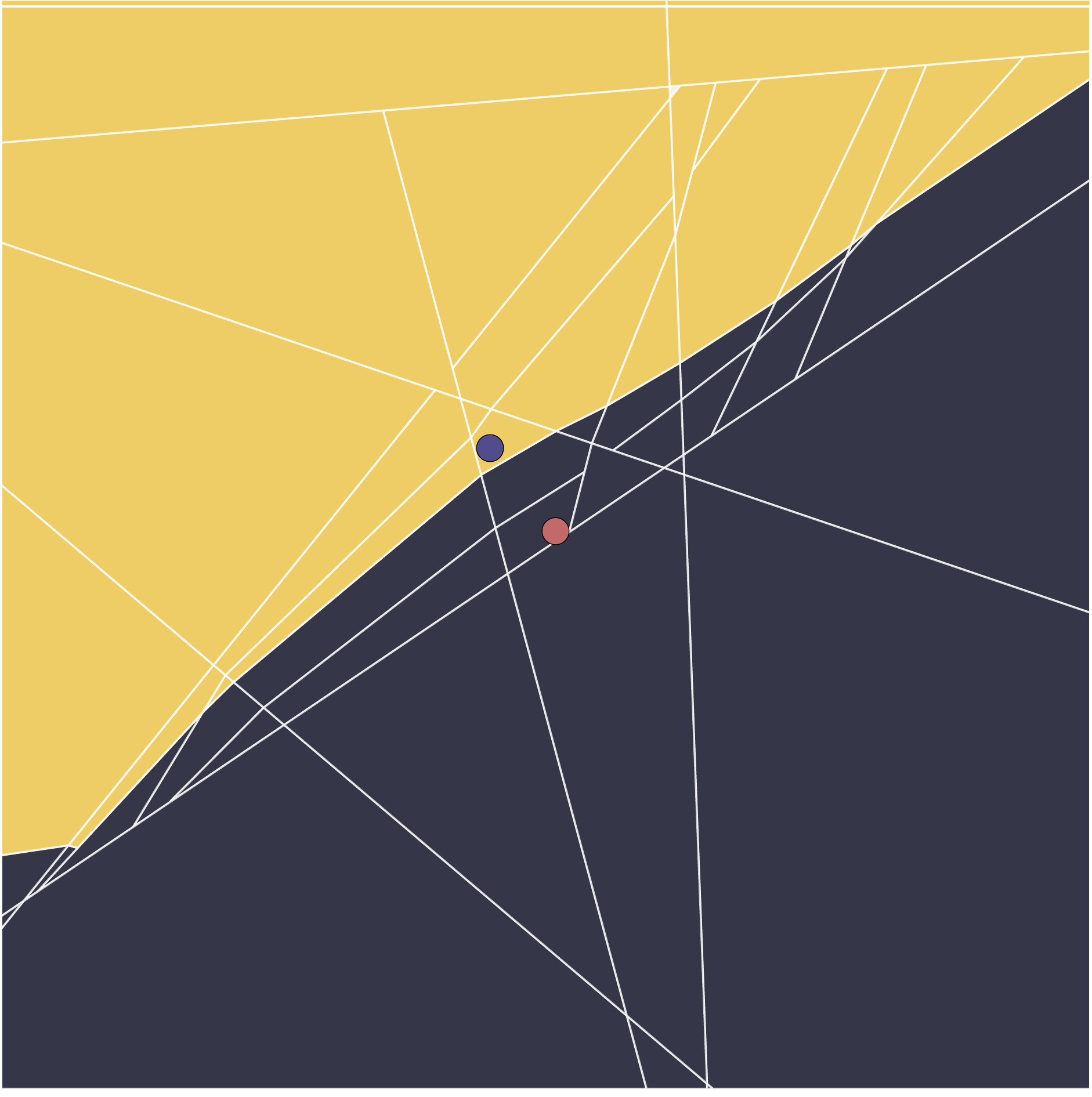}
	\caption{A function plot representing the TADS of $\nu'_c$ in an area around $\vec{x}_9$. The yellow area contains points that are classified correctly, the black area contains points that are classified incorrectly. The blue point represents $\vec{x}_9$ and the red point represents the adversarial example $\vec{x}_5$.}
	\label{fig:neighborhood_plot}
\end{figure}

\section{Experimental Results}\label{subsec:experimental_results}
In the following, we will showcase experimental results regarding the TADS-based verification of neural networks using PCA to reduce the dimensionality of the verification problem. We will start by considering the reduction to two dimensions, allowing us to visualize the process and showcase its workings conceptually.
Afterwards, we will move towards higher dimensions, examining more concrete questions of scalability. 

\subsection{Conceptual Showcase and Visualization}
\label{subsec:2dtads}
\begin{figure*}[t]
	\centering
	\begin{tabular}{>{\centering\arraybackslash}m{.2\textwidth}m{.5in}>{\centering\arraybackslash}m{.2\textwidth}m{.1in}>{\centering\arraybackslash}m{.2\textwidth}}
		\centering\arraybackslash
		\includegraphics[width=.2\textwidth]{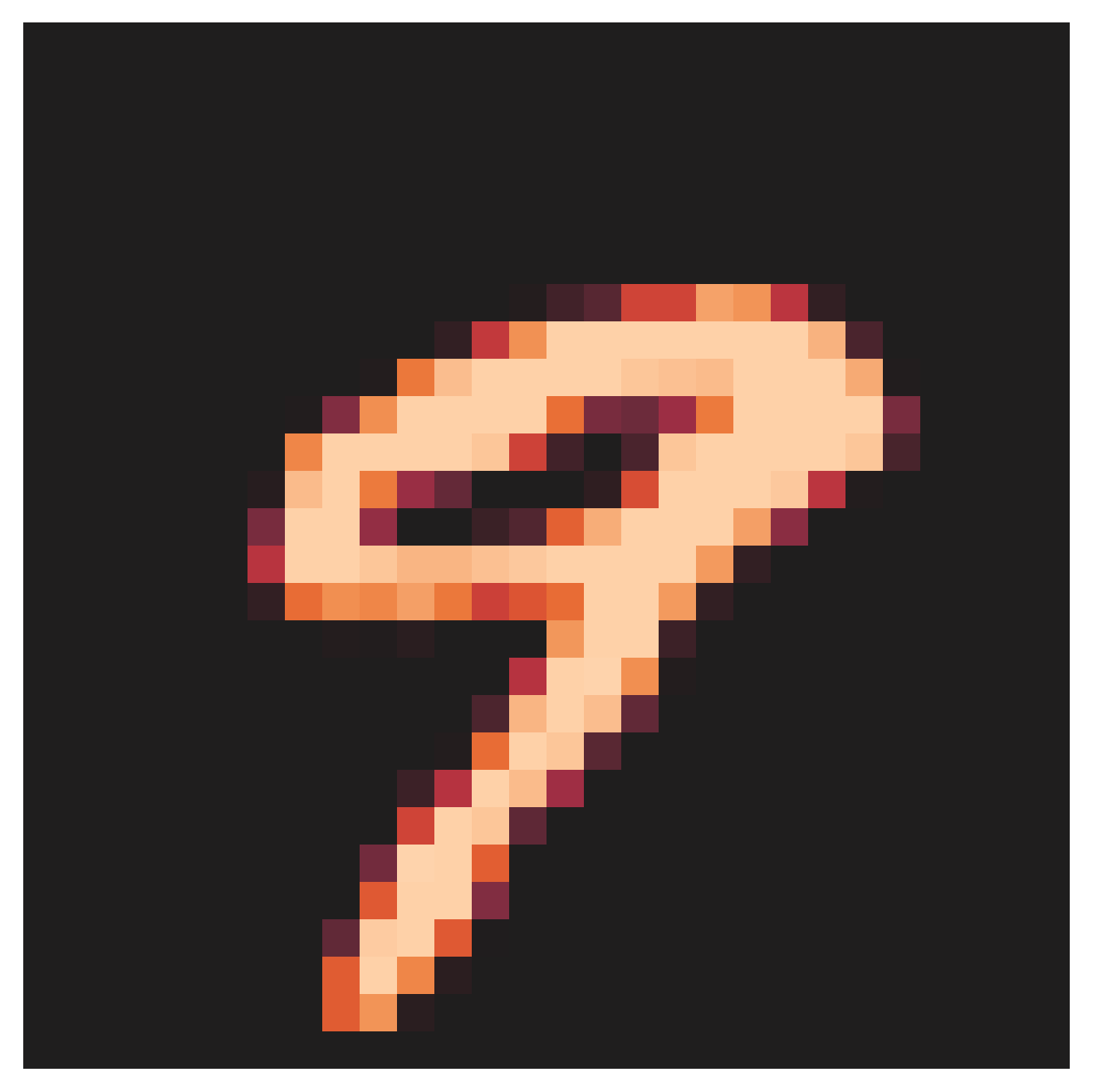}
		&%
		\centering\arraybackslash%
	$\ +\ .01\ \times$ &%
		\includegraphics[width=.2\textwidth]{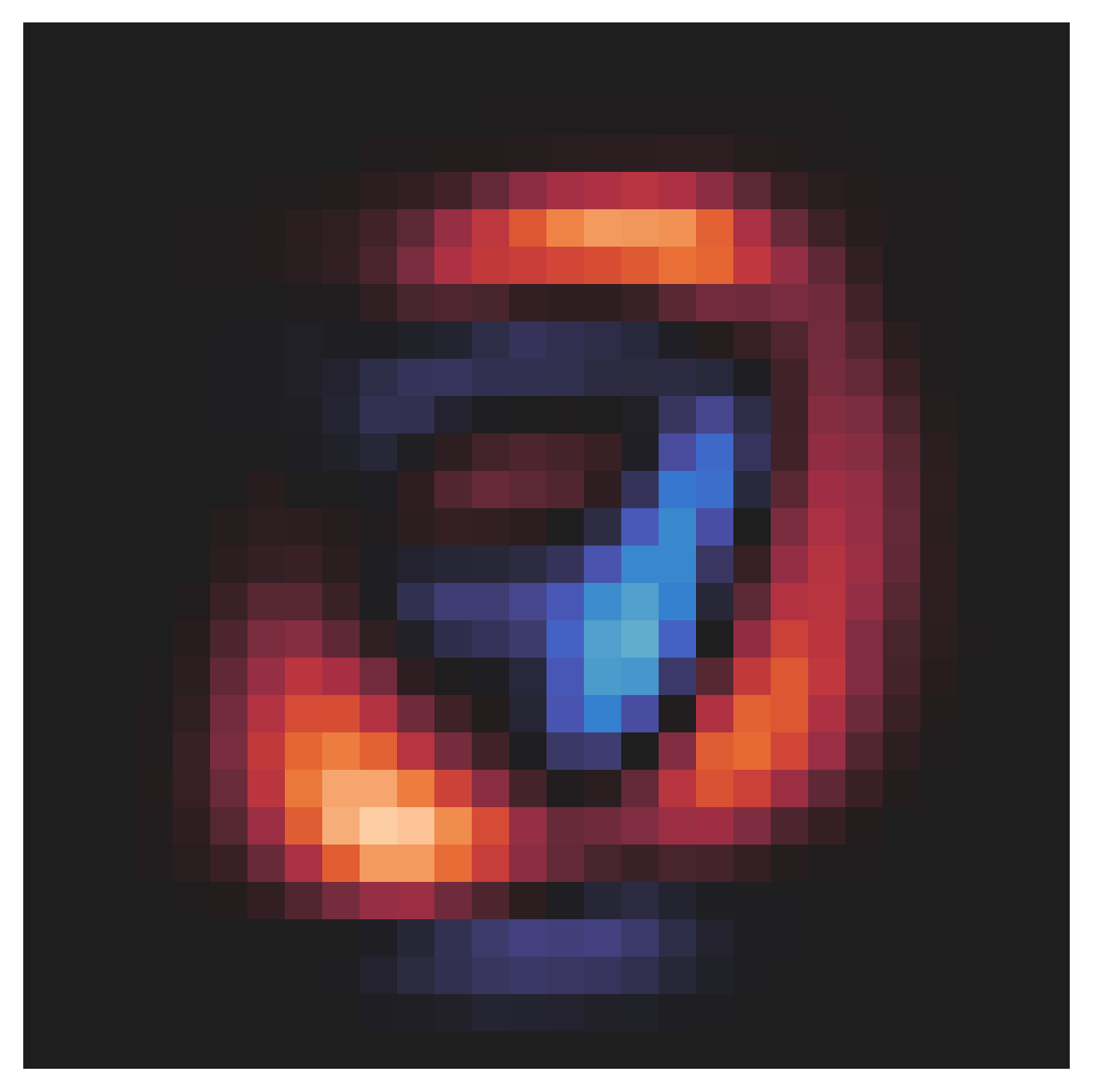}
		&%
		$=$ & %
		\includegraphics[width=.2\textwidth]{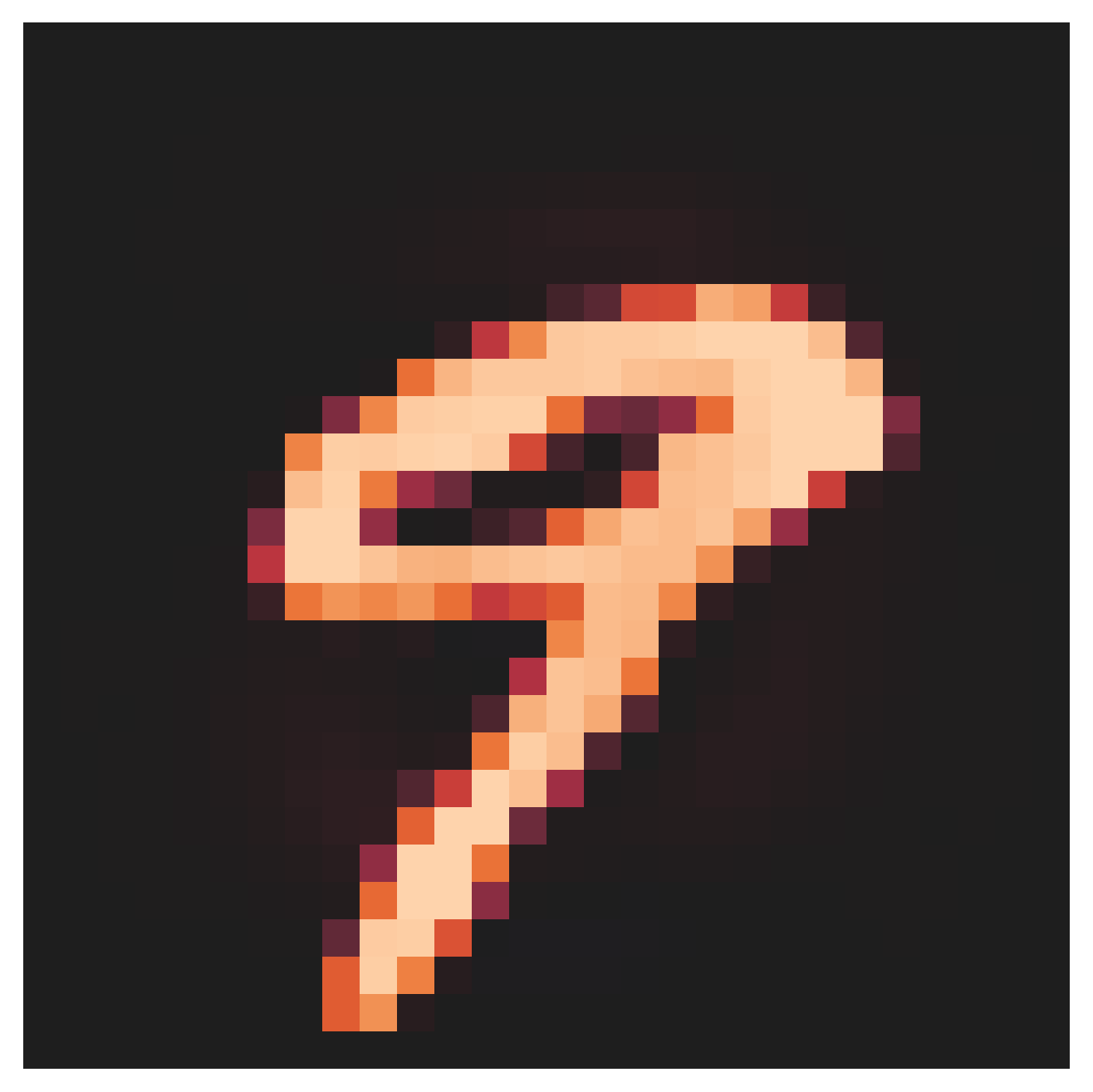}
		\\
		$\nu$: ``9'' & & $\Delta$ &   & $\nu$: ``5''
	\end{tabular}
	\caption{MNIST sample that represents the number ``9'' (left) and a close adversarial example that is classified as ``5'' (right).
	The difference between the two is marginal (center).
	The adversarial was found in a neighboring linear region using a restricted TADS (cf.,\ \cref{fig:2d_tads}).
	Vectors are visualized using a perceptually uniform diverging color palette (Seaborn's ``icefire'').
	Idea of representation \cite{goodfellow2014explaining}.}
	\label{fig:2dsample}
\end{figure*}
\begin{figure}
\center
\includegraphics[scale=0.2,angle =270]{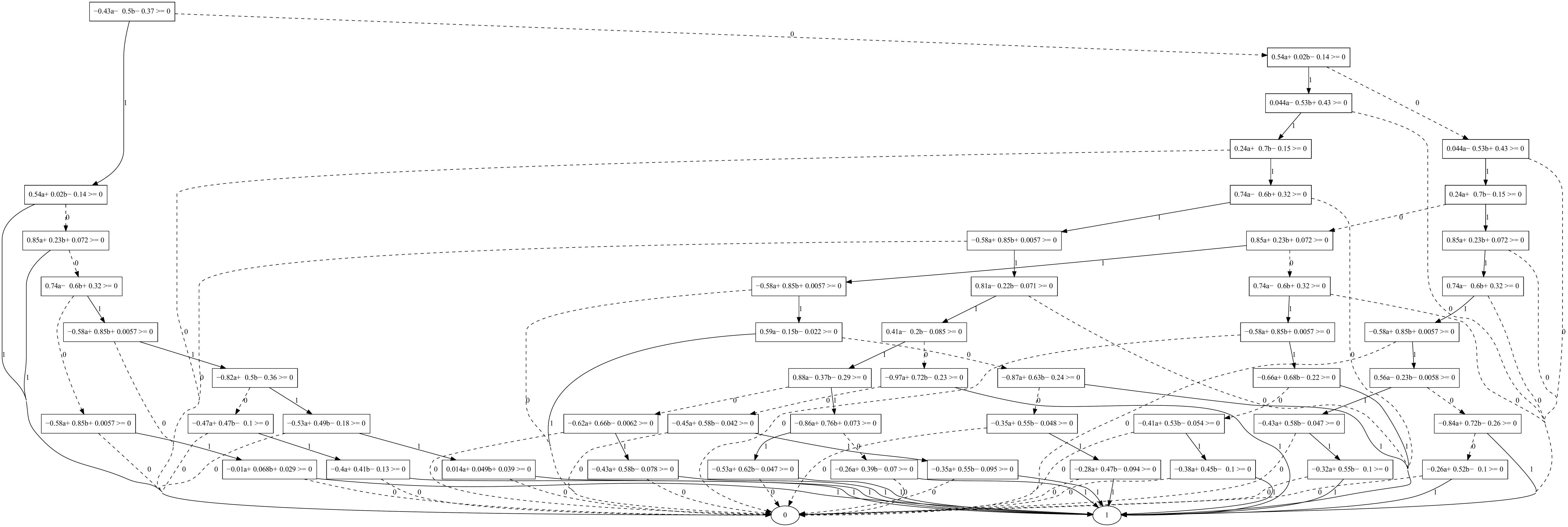}
\caption{A TADS representing the behavior of $\nu'$ around $\vec{x}_9$. For readability, this TADS is constructed such that $\vec{x}_9$ corresponds to the vector $(0,0)$. The terminal node ``1'' represents a correct classification, the node ``0'' an incorrect one.}
\label{fig:2d_tads}
\end{figure}
For this section, we consider the neural network classifier
\[ \nu_c = {\argmax }\circ {\netsem{\nu'}} \circ \rho_2 \]
where $\nu'$ is a fully connected ReLU-network with 5 layers of 10 neurons each. 
Training is done on the MNIST training set with batches of 300 images per training step using standard settings of the ADAM optimizer \cite{kingma2014adam}.
This classifier uses the two-dimensional PCA representation. This allows us to plot the function represented by $\netsem{\nu'} \circ \rho_2$ as done in \cref{fig:neighborhood_plot}. 

We consider the sample $\vec x_9$ shown in \cref{fig:2dsample}. This image is  classified correctly by $\nu_c$, being assigned the label ``9''. However, as we will see, this classification is very unstable.

Using TADS, we can gain insight into this prediction by creating the class characterization TADS 
\[ t^9_{\nu_c} = t_{\nu'} \tadsjoin t_a \tadsjoin t_{x=9} \]
for $\nu_c'$ and class ``9''  on the infinity ball $\vec{x}_9 + 0.3 \cdot B^{2}_{\infty}$
This TADS is shown in \cref{fig:2d_tads} and can be interpreted as follows:
\highlight{Any input belonging to the set $\vec{x}_9 + 0.3 \cdot B^{2}_{\infty}$ that reaches the ``1'' terminal in the TADS depicted in \cref{fig:2d_tads} is classified as a ``9'' by $\nu_c$ (which is the desired behavior), while all others are adversarial examples.}

Moreover, we can visualize the function plot corresponding to this TADS as shown in \cref{fig:neighborhood_plot}. Note that lines in this plot indicate decision boundaries that are implied by the non-terminal nodes in the TADS\@. These decision boundaries separate the regions of the piece-wise affine function encoded by the neural network. 
As a consequence, each polygon that is enclosed by such linear boundaries corresponds to precisely one path in the TADS $t^9_{\nu_c}$.

One can immediately observe that while $\nu_c$ classifies $\vec{x}_9$ correctly, there exists a close region of inputs that are classified incorrectly. Using the information contained in the TADS, it is trivial to obtain adversarial examples by picking any path in the TADS ending in the ``0'' terminal and finding a point satisfying the corresponding path condition. An example adversarial example generated in this way is shown in \cref{fig:2dsample}. Observe that, while being classified differently by $\nu_c$, both images are almost identical to the human eye, which indicates that this neural network might not be entirely trustworthy even though it classified $\vec{x}_9$ correctly.

\subsection{Scaling to higher dimensions}
\begin{figure}
\centering
\includegraphics[scale=0.4]{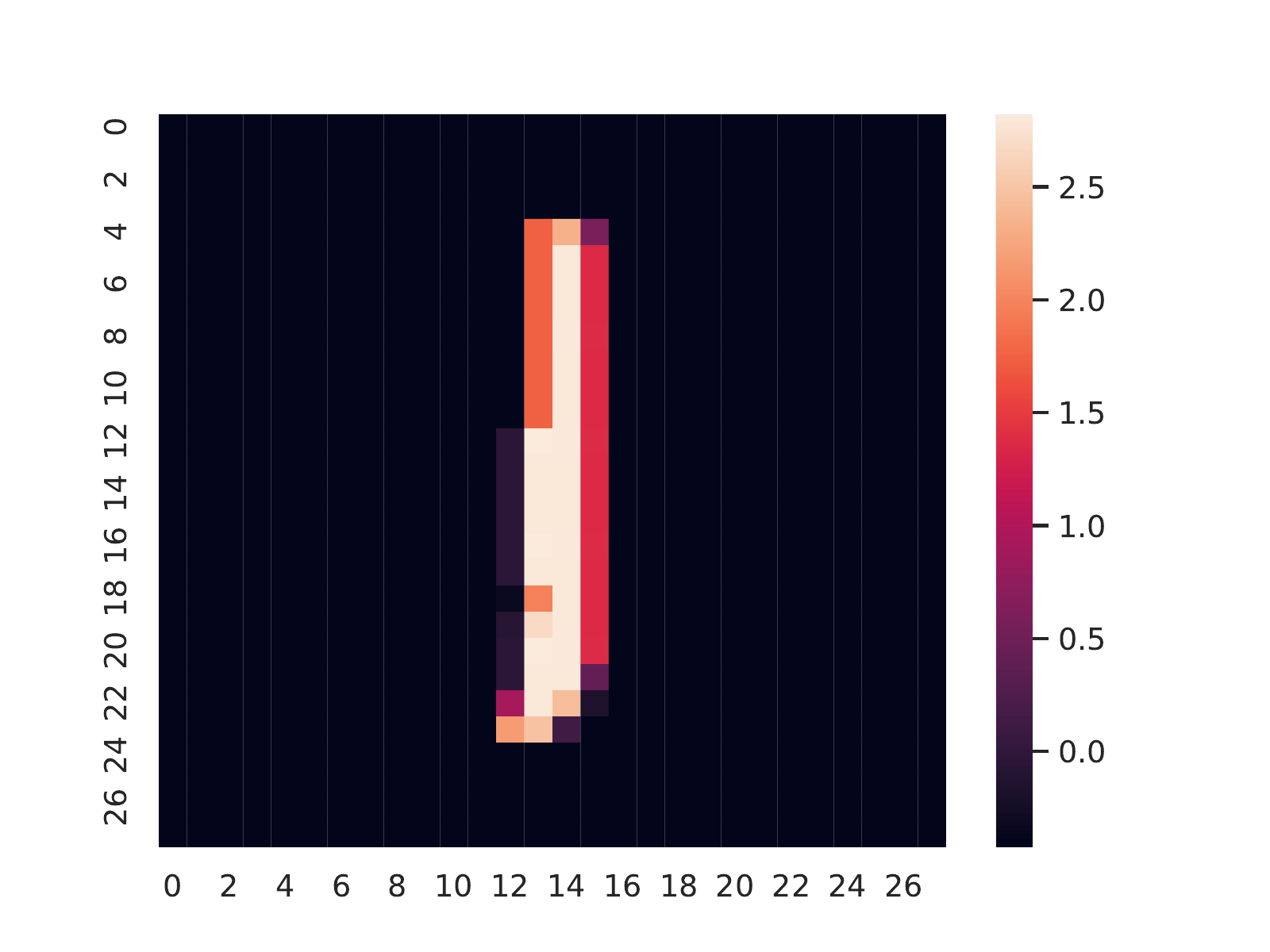}
\caption{An MNIST sample image $\vec x_1$ representing the digit ``1''.}
\label{fig:advsample}
\end{figure}

After showcasing our verification approach conceptually on a 2-dimensional problem, we now move towards higher dimensions and seek to examine how the addition of new dimensions affects scalability.
To do this, we construct a neural network classifier that uses a 6-dimensional input representation instead of a 2-dimensional one (all other settings are equal)
\[ \nu_c = \underbrace{{\argmax } \circ { \sem{\nu'}}}_\text{classifier} {}\circ \rho_6 \eqeos \]
This increase of dimension drastically improves the accuracy, however at the price of an explosion in size of the corresponding TADS\@. All reported numbers reflect an average according to six random runs.

\paragraph{Accuracy.} The six-dimensional neural network classifier $\nu_c$ achieved roughly 74\% accuracy on the test set in comparison
to the 91\% accuracy of the original unrestricted network, but much better than the 46\% accuracy of the two-dimensional classifier
(cf., \cref{traineddracc}). 

We also tested different dimensions for PCA with respect to network accuracy, the results of which can be found in \cref{fig:PCAaccplot}.
These results show that in this case, a dimensionality reduction by an order of magnitude still allows one to achieve 90\% accuracy,
which is very close to the 91\% accuracy of the original network.

\paragraph{Scalability.} In the two-dimensional case, we showed an
example where robustness around some input could be accurately
disproven with a radius of $\delta=0.3$, which, according to \cref{lem:pca_robustness} implies 
\[ \epsilon=\frac{\delta}{\max_i \norm{p_i}_1} \approx\frac{0.3}{20}=0.015 \]
robustness of the 785-dimensional network.\footnote{The denominator of 20 was computed as the maximum of the $l_1$ norms of the first six principal components of MNIST.} 
The corresponding TADS, describing network behavior in the space of interest, had 51 nodes and
could be handled quite easily. We repeat this experiment with the input image shown in \cref{fig:advsample} and the six-dimensional neural network instead. The TADS resulting from this experiment possesses roughly 4600 nodes.
This is still manageable computationally, but indicates the expected
explosion in size.

\section{Related Work}
\label{sec:related work}
The topic of robustness has been widely discussed in the machine learning community ever since it first gained attention in 2013~\cite{szegedy2013intriguing}. One topic of interest is research into heuristic methods that quickly and reliably find adversarial examples for modern neural networks, serving to understand how adversarial examples occur and therefore how they might be mitigated~\cite{carlini2017towards,goodfellow2014explaining}. As they are devised by the machine learning community, it is not surprising that these methods devised to find adversarial examples typically leverage methods from the machine learning toolbox, using gradient descent and other training heuristics to find adversarial examples.

Another natural topic with respect to robustness has been constructing neural networks that are reliably robust after training. A typical approach to this is defensive distillation~\cite{papernot2016distillation}. Defensive distillation seeks to secure a previously trained neural network. This is achieved by using the outputs of the first neural network to train a second neural network with equivalent architecture. This process is called distilling. The additional information provided by the first neural network allows for efficient training in less training steps, reducing the need for large parameter values and therefore reducing the risk of adversarial attacks. Other approaches directly modify the training process to ensure scalability, usually by introducing additional regularization terms that are meant to steer the training process into a robust direction, often working in tandem with formal methods~\cite{zheng2016improving,wang2018mixtrain,han2018co}.

Closer to our approach are neural network verification approaches (for robustness). They can be split into two categories, approaches based on branch-and-bound tree search algorithms \cite{dakin1965tree,leofante2018automated} and approaches based on abstract interpretation~\cite{cousot1992abstract}.

\paragraph{Neural Network Verification --- Tree Search. } 
Much like SAT and SMT solvers, these approaches use a branch-and-bound tree search algorithm to find a counterexample to the property of interest. A critical part of this is finding an apt ReLU configuration, i.e., which neuron activation values need to be set to 0 by the ReLU activation function and which do not. This corresponds to finding a satisfiable path in a TADS that contains a counterexample, which makes TADS based verification inherently a representative of this category.

Other examples include Reluplex~\cite{katz2017reluplex}, one of the earliest scalable neural network verifiers, and alpha-beta-crown~\cite{wang2021beta}, a modern method that can be regarded as current state-of-the-art \cite{bak2021second}. Methods of this type differ mostly in the heuristics that guide their branching and bounding.

Tree search methods are accurate and leading in practice, but they tend to be more time intensive than abstract-interpretation based methods. Moreover, they are, much like TADSs, naturally restricted to piece-wise affine neural networks and cannot cover activation functions such as sigmoid or softmax. 

\paragraph{Neural Network Verification --- Abstract Interpretation. }
These neural network verifiers define an abstract interpretation of neural networks to attain an overapproximation of the reachable states that a neural network can output on a given input region \cite{elboher2020abstraction}. As these methods compute an overapproximation of the truly reachable states, they are safe, but not complete, i.e., they might incorrectly state that a given property is violated when it is not. On the flipside, abstract interpretation verifiers are typically computationally quite efficient and extend to neural networks that are not piece-wise affine. Examples of verifiers based on abstract interpretation include AI$^2$~\cite{elboher2020abstraction} and DeepPoly~\cite{singh2019abstract}. Our TADS-based approach naturally also applies to abstractly interpreted neural networks.

\section{Conclusion}
In this paper, we have applied TADS, a white-box representation of neural networks, to the problem of neural network robustness.
To apply TADS to this problem, we have introduced precondition projection and showed how to extend the argmax function, that is typically used with neural networks in classification tasks, to generate TADS that precisely describe a neural networks classification behavior in a given area around a fixed input point.
Choosing the considered robustness region as precondition, robustness becomes equivalent to the property that the entire corresponding TADS collapses to one node that then characterizes the robust classification.
If this is not the case, the resulting TADS explicitly represents the set of all adversarial examples. 

This unique power of TADS-based robustness verification comes at the price of an exponential complexity which we have proposed to mitigate via PCA-based dimensionality reduction
by focusing the verification on the image of a low-dimensional PCA encoding. Three versions of this approach have been discussed:
\begin{itemize}
	\item 
	An approximative version that can be regarded as an elaborate search heuristics for adversarial examples,
	\item
	A transformational approach where the PLNN is extended by a preprocessing step defined by PCA-based auto-encoding, 
	and which allows one to infer robustness of the 784-dimensional transformed network based on the analysis of 
	the corresponding low-dimensional PCA space, and 
	\item
	An approach that is based on a modified learning process, specifically tailored to the corresponding PCA-based encoding. This method leverages the machine learning toolbox to improve the accuracy of the 784-dimensional transformed network while still allowing low-dimensional robustness verification.
\end{itemize}
We believe that dimensionality reduction, as illustrated in this paper for PCA, is key to achieve neural networks that are ready for verification.
The challenge is to find dimensionality reduction techniques that maintain a high level of accuracy.
In our experience, the success of such techniques hinges on characteristics of the application domain.
We are optimistic that this approach will widen the scope of applications where neural networks are accepted.

\bibliographystyle{alpha}
\bibliography{literature}   

\end{document}